\documentclass[journal,comsoc]{IEEEtran}
\usepackage{amsmath,amsfonts}
\usepackage{algorithmic}
\usepackage{algorithm}
\usepackage{array}
\usepackage{textcomp}
\usepackage{stfloats}
\usepackage{url}
\usepackage{verbatim}
\usepackage{graphicx}
\usepackage[normalem]{ulem}

\IEEEoverridecommandlockouts
\usepackage{cite}
\def\BibTeX{{\rm B\kern-.05em{\sc i\kern-.025em b}\kern-.08em
    T\kern-.1667em\lower.8ex\hbox{E}\kern-.125emX}}

\usepackage{amssymb}
\usepackage{amsthm}
\usepackage{nicefrac}
\usepackage{xcolor}
\usepackage{caption}
\usepackage{subcaption}
\usepackage{mathtools,cuted}
\usepackage{multicol}

\usepackage{tikz}
\usetikzlibrary{positioning,arrows.meta,circuits,calc}
\tikzset{%
	block/.style    = {draw, thick, rectangle},
	mult/.style     = {draw, circle}, 
	sum/.style      = {draw, circle}, 
	input/.style    = {coordinate}, 
	output/.style   = {coordinate} 
	label/.style    = {draw=none} 
	intercon/.style = {sensor, rounded corners},
	Dotted/.style={
		line width=1.2pt,
		dash pattern=on 0.01\pgflinewidth off #1\pgflinewidth,line cap=round,
		shorten >=0.5em,shorten <=0.5em},
	Dotted/.default=12 
}
\usepackage{adjustbox}

\newtheorem{theorem}{Theorem}
\newtheorem{lemma}[theorem]{Lemma}
\newtheorem{proposition}[theorem]{Proposition}
\newtheorem{definition}{Definition}
\newtheorem{example}{Example}

\newtheorem{problem}{Problem}
\newtheorem{remark}{Remark}

\newtheorem{innercustomgeneric}{\customgenericname}
\providecommand{\customgenericname}{}
\newcommand{\newcustomtheorem}[2]{%
	\newenvironment{#1}[1]
	{%
		\renewcommand\customgenericname{#2}%
		\renewcommand\theinnercustomgeneric{##1}%
		\innercustomgeneric
	}
	{\endinnercustomgeneric}
}
\newcustomtheorem{customproblem}{Problem}

\DeclareMathOperator{\sign}{sign}
\newcommand{\bs}[1]{\boldsymbol{#1}}

\newcommand{\yb}[1]{{\color{blue}#1}}
\newcommand{\ybc}[1]{{\color{green}YB: #1}}

\newcommand{\yc}[1]{{\color{purple}#1}}

\usepackage[left=1.1in,right=1.1in,top=1.1in,bottom=1.1in]{geometry}

\begin{document}

\title{Robust Regression with Ensembles \\ Communicating over Noisy Channels}

\author{Yuval Ben-Hur,~\IEEEmembership{Graduate Student Member,~IEEE,} and Yuval Cassuto,~\IEEEmembership{Senior Member,~IEEE}
}



\maketitle

\begin{abstract}
    As machine-learning models grow in size, their implementation requirements cannot be met by a single computer system. 
    This observation motivates distributed settings, in which intermediate computations are performed across a network of processing units, while the central node only aggregates their outputs.
    However, distributing inference tasks across low-precision or faulty edge devices, operating over a network of noisy communication channels, gives rise to serious reliability challenges.
    We study the problem of an ensemble of devices, implementing regression algorithms, that communicate through additive noisy channels in order to collaboratively perform a joint regression task.
    We define the problem formally, and develop methods for optimizing the aggregation coefficients for the parameters of the noise in the channels, which can potentially be correlated. Our results apply to the leading state-of-the-art ensemble regression methods: bagging and gradient boosting. We demonstrate the effectiveness of our algorithms on both synthetic and real-world datasets.  
\end{abstract}


\section{Introduction}


\IEEEPARstart{M}{achine} learning (ML) and artificial intelligence (AI) algorithms have been game changers across many technology fields over the past decade, achieving unprecedented performance in extremely complex tasks. 
ML models, such as deep neural networks (DNNs) that are a major thrust of this revolution, are driven mainly by their representation power. Complicated tasks addressed nowadays, such as image processing or natural language processing (NLP), require large models with billions of parameters. The exponential growth in model size has made training and inference incredibly
challenging computational tasks.

The meteoric success of ML/AI has driven increasing demands for computing hardware, which is needed for improving performance and for extending into more challenging applications. However, it is becoming clear that the exponentially growing requirements of state-of-the-art models cannot be met by hardware scaling rates~\cite{gholami2021ai}. This motivates an increasing interest in \emph{distributed ML/AI} that allow flexible and elastic resource provisioning across multiple locations and entities. Distributed training and inference have two additional advantages stemming from the computations they perform on local data sets: reducing communication costs and improving privacy. 

A natural framework for distributed ML/AI is \emph{ensemble} learning algorithms. In this framework, the learning task is shared by multiple \emph{base predictors}, whose outputs are \emph{aggregated} at inference time to obtain a strong prediction. Ensemble methods have already been proven to achieve state-of-the-art performance~\cite{garcia2005cooperative} (and many more). In addition to their applicability to distributed inference over communication networks~\cite{kim2012target,he2021efficient}, they are compatible with common AI chip architectures~\cite{chen2020survey}, so can potentially enrich the algorithmic stack in existing ML/AI chips.

This paper studies distributed \emph{inference} over noisy channels, and in particular the task of \emph{regression}.  Regression seeks to fit a functional model to the relationship between a dependent variable and one or more independent variables called features. A regression model is usually trained from realizations of features and the corresponding, possibly inaccurate, values of the dependent variable.  \emph{Ensemble regression} is a powerful and elegant technique for constructing models that are aggregates of multiple base regressors. Such models play a key role in a wide range of real-world problems and applications~\cite{sagi2018ensemble}: from big-data analysis (e.g., time-series forecasting~\cite{qiu2014ensemble}, outlier detection~\cite{kaneko2018automatic}) to estimation problems (e.g., source localization~\cite{kim2012target}, image  alignment~\cite{kazemi2014one}) and beyond. Further, regression may be regarded as a form of \emph{data compression} from redundant features to task-relevant variables, thus ensemble regression may be used for distributed \emph{task-oriented source coding}. This success and promise notwithstanding, when deploying ensemble regression models over noisy channels, the noise added in transit between the base regressors and the aggregating node may severely compromise the accuracy, and thus requires mitigation. The question, which we answer in this paper to the affirmative, is can we mitigate this degradation if we know the statistics of the noise added to the regressor outputs?

Recent attempts to implement learning ensembles on real hardware devices have encountered serious reliability issues due to noise~\cite{gao2022soft,kim2012target,he2021efficient} . In~\cite{gao2022soft}, for example, an ensemble of convolutional neural networks (CNN), used for image classification and implemented on FPGA accelerators, suffered from performance degradation due to errors induced by radiated noise. 
In~\cite{kim2012target} and~\cite{he2021efficient}, neural network ensembles were deployed on wireless sensor networks for localization and navigation purposes. The network comprised low-precision edge devices that produce unreliable individual predictions, resulting in degraded overall prediction performance.
A similar noisy setting also exists in neural-network hardware accelerators, wherein analog compute engines~\cite{haensch2018next} aggregate neuron inputs contaminated by computation noise and non-linearities.

Most prior work dealing with noise in the ensemble machine-learning literature has focused on \emph{noisy training data}~\cite{cizek2020robustreview} (e.g., model mixture contamination~\cite{du2018robust}, or outlying samples~\cite{blatna2006outliers}). Distributed inference ensembles plagued by channel noise have received attention recently, but mostly focused on the \emph{classification} problem~\cite{kim2020distributed, cassuto2021boosting, benhur2021distributed, shao2021learning, ben2023ensemble}. The problem of noisy ensemble regression has not received a prior systematic study, to the best of our knowledge. Our recent short paper~\cite{benhur2022regression} identifies the problem of using standard ensemble algorithms with communication noise, and suggests some basic optimization framework for the \emph{mean squared error (MSE)} loss. The results of this present paper significantly extend that initial framework into general loss functions, such as \emph{mean absolute error (MAE)} and general $\ell_p$ norms, apply to more ensemble algorithms, and include analytical results such as lower and upper error bounds. Earlier works addressed reliability challenges of ensemble regression using heuristic or ad-hoc solutions: from increasing the ensemble size~\cite{he2021efficient}, to heuristically re-weighting~\cite{kim2012target}, and even completely ignoring~\cite{gao2022soft} the outputs of individual base regressors.


Our study considers the two leading state-of-the-art ensemble regression methods: \emph{bagging}~\cite{breiman1996bagging}, and \emph{gradient boosting}~\cite{ friedman2001greedy, friedman2002stochastic}. The former trains each base regressor on a different sample from the data set, and the latter iteratively augments the ensemble with a base regressor that trains on the errors of the current ensemble members. A key feature of our proposed framework is that the noise-mitigating algorithms are developed around, and not as replacement, of the existing ensemble algorithms. This implies that any progress made in algorithms driving these methods can also benefit their deployment in noisy settings. Toward that, our optimization target for the noisy setting is the \emph{aggregation coefficients} that weight the base regressors' contributions to the final regression. While bagging and gradient boosting have a structurally identical regression function, the difference in how they are trained dictates different noise mitigation: in bagging each regressor is trained independently, so the aggregation coefficients can be calculated ``post-training'' after all base regressors are known, while in gradient boosting each aggregation coefficient is set as part of a training iteration, so the noise-mitigating optimization needs to be done within the training sequence.

The results of the paper are organized as follows. In Section~\ref{sec:problem-model-and-motivation}, we formally define the noisy ensemble-regression problem and motivate it by illustrating the effect of channel noise on the final ensemble prediction.
In Section~\ref{sec:robust-bagging}, we address robust regression for \emph{bagging ensembles}. We start with results for general $\ell_p$-norm loss, and then propose optimization algorithms for the MSE and MAE loss functions. For MSE we give a closed-form expression for the aggregation coefficients in a generalized optimization problem that allows weighting the channel noise vs. the model error. For MAE we derive analytical results for normally distributed noise that enable a gradient-based optimization algorithm, as well as error upper and lower bounds.  
Section~\ref{sec:robust-gradboost} moves to address robust \emph{gradient boosting}, where the main contribution is a training algorithm that sets the aggregation coefficients to optimize the \emph{expected loss} in the noisy setting. For the MSE we derive a closed-form expression for the optimal coefficients. In section~\ref{sec:experiments}, we present an empirical study of the proposed methods for synthetic and real-world datasets. The results show that the proposed algorithms significantly outperform baseline bagging aggregation algorithms for both MSE and MAE. For gradient boosting the results show that the noise-mitigated algorithm enables a desired property of decreasing the error as the ensemble size grows, while the baseline algorithm exhibits increasing error once deployed with noise. The paper is finally concluded in Section~\ref{sec:conclusion}.

\section{Distributed noisy regression}\label{sec:problem-model-and-motivation}
\subsection{Model formulation}\label{sec:model-formulation}
Consider an unknown (deterministic) target function $f(\cdot): \mathbb{R}^u \rightarrow \mathbb{R}$ that is required to be estimated given a set of data samples $\mathcal{S} = \{(\bs{x}_i,y_i)\}_{i=1}^{N_s}$, where the feature vector is $\bs{x}_i \in \mathbb{R}^{u}$ ($u\in\mathbb{N}$) and its corresponding value is $y_i \in \mathbb{R}$.
We denote the sets of $\bs{x}_i$ and $y_i$ from which the dataset $\mathcal{S}$ is comprised as $\mathcal{X}$ and $\mathcal{Y}$, respectively.
A base regressor (or sub-regressor) $\varphi(\cdot): \mathbb{R}^u \rightarrow \mathbb{R}$ is a function aimed to estimate $f(\cdot)$, based on $\mathcal{S}$ and a prescribed loss function $\delta(\cdot,\cdot)$ (measuring dissimilarity between $y_i$ and $\varphi(\bs{x}_i)$) as a figure of merit.
In \emph{ensemble} regression, the regressor is implemented by aggregating predictions from multiple base regressors $\{\varphi_t(\cdot)\}_{t=1}^{T}$, as now defined. Throughout the paper, $T$ denotes the \emph{ensemble size}.
\begin{definition}[regression ensemble]\label{def:reg-ensemble}
    Define a regression ensemble of size $T\in\mathbb{N}$ as a set of functions $\{\varphi_t(\cdot)\}_{t=1}^{T}$ where $\varphi_t: \mathbb{R}^u \rightarrow \mathbb{R}$ and an aggregation function $\hat{f}: \mathbb{R}^{T} \rightarrow \mathbb{R}$.
\end{definition}

The ensemble's prediction for a data sample $\bs{x}$ is obtained by evaluating the ensemble aggregation function on the individual predictions of the base regressors, i.e., $\hat{f}(\varphi_1(\bs{x}),\dots,\varphi_T(\bs{x}))$.
In this paper, we consider an ensemble regressor implemented in a distributed fashion over \emph{noisy channels}, a setting that consists of the following components:
\begin{enumerate}
    \item A set of individual computation nodes implementing the base regressors $\{\varphi_t(\cdot)\}_{t=1}^{T}$.
    \item Noisy communication channels $\mathcal{C}_t: \mathbb{R} \rightarrow \mathbb{R}$ connecting each computation node to a central processor.
    \item A central processor that produces the final prediction by applying the ensemble aggregation function $\hat{f}(\tilde{\varphi}_1,\dots,\tilde{\varphi}_T)$ where $\tilde{\varphi}_t \triangleq \mathcal{C}_t(\varphi_t(\cdot))$. 
\end{enumerate}
This setting, referred to as \emph{noisy ensemble regression}, is depicted in Fig.~\ref{fig:sys-blk-diag} for additive noise channels $\mathcal{C}_t: \tilde{\varphi}_t(\cdot) = \varphi_t(\cdot) + n_t$, where $n_t$ is a random variable representing the additive noise.
Correspondingly, the aggregated noisy prediction is denoted $\tilde{f}(\bs{x})$. 

\begin{figure*}
    \centering
    \begin{adjustbox}{width=0.8\textwidth}
        \begin{tikzpicture}[>=latex]
            \node (data-in) at (0,0) {$\bs{x}$};
            \node [above=of data-in] (above-data-in) {};
            \node [below=of data-in] (below-data-in) {};
            
            \node [block, right=of above-data-in] (base-1) {            \begin{tabular}{c}
                Base regressor \\ \#$1$
            \end{tabular}
                };
                
            \node [block, right=of below-data-in] (base-T) {
            \begin{tabular}{c}
                Base regressor \\ \#$T$
            \end{tabular}
                };
                
            \node (mid-base1w-baseTw) at ($(base-1.west)!0.5!(base-T.west)$) {}; 
            \node (mid-data-in-mid-base1w-baseTw) at ($(data-in.east)!0.33!(mid-base1w-baseTw.center)$) {}; 
            
            \draw[-,thick] (data-in.east) -- (mid-data-in-mid-base1w-baseTw.center);
            
            \node[left=of base-1.west, above=of mid-data-in-mid-base1w-baseTw] (upper-left-base-1-mid) {};
            \draw[-,thick] (mid-data-in-mid-base1w-baseTw.center) -- (upper-left-base-1-mid.center);
            \draw[->,thick] (upper-left-base-1-mid.center) -- (base-1.west);
        
            \node[left=of base-T, below=of mid-data-in-mid-base1w-baseTw] (lower-left-base-T-mid) {};
            \draw[-,thick] (mid-data-in-mid-base1w-baseTw.center) -- (lower-left-base-T-mid.center);
            \draw[->,thick] (lower-left-base-T-mid.center) -- (base-T.west);
            
            \draw[Dotted] (base-1.south) -- (base-T.north); 
                        
            \node [sum, right=of base-1.east] (add-1) {$+$};
            \node [sum, right=of base-T.east] (add-T) {$+$};
            \draw[->,thick] (base-1.east) -- node[above] {$\varphi_1(\bs{x})$} (add-1.west);
            \draw[->,thick] (base-T.east) -- node[above] {$\varphi_T(\bs{x})$} (add-T.west);
            
            \node [above=of add-1] (n-1) {$n_1$};
            \node [above=of add-T] (n-T) {$n_T$};
            \draw[->,thick] (n-1) -- (add-1);
            \draw[->,thick] (n-T) -- (add-T);
            
            \node [block, right=of add-1.east] (equalize-1) {$\times \alpha_1$};
            \node [block, right=of add-T.east] (equalize-T) {$\times \alpha_T$};
            \draw[->,thick] (add-1.east) -- node[above] {$\tilde{\varphi}_1(\bs{x})$} (equalize-1);
            \draw[->,thick] (add-T.east) -- node[above] {$\tilde{\varphi}_T(\bs{x})$} (equalize-T);
            
            \node (mid-add1-addT) at ($(equalize-1.east)!0.5!(equalize-T.east)$) {}; 
            \node[block,right=of mid-add1-addT,minimum height = 4cm] (weighted-sum) {
            \begin{tabular}{c}
                    Central\\Processor
                    \\
                    \\
                    $\sum_{t=1}^{T} \alpha_t \tilde{\varphi}_t(\bs{x})$
                \end{tabular}
            };

            \draw[->,thick] (equalize-1) -- node[above] {} (add-1-|weighted-sum.west);
            \draw[->,thick] (equalize-T) -- node[above] {} (add-T-|weighted-sum.west);
        
            \node [right=of weighted-sum] (output-class) {$\tilde{f}(\bs{x})$};
            \draw[->,thick] (weighted-sum) -- (output-class);
        
        \end{tikzpicture}
    \end{adjustbox}
    \caption[justification=centering]{Block diagram of a noisy ensemble regression system.}
    \label{fig:sys-blk-diag}
\end{figure*}
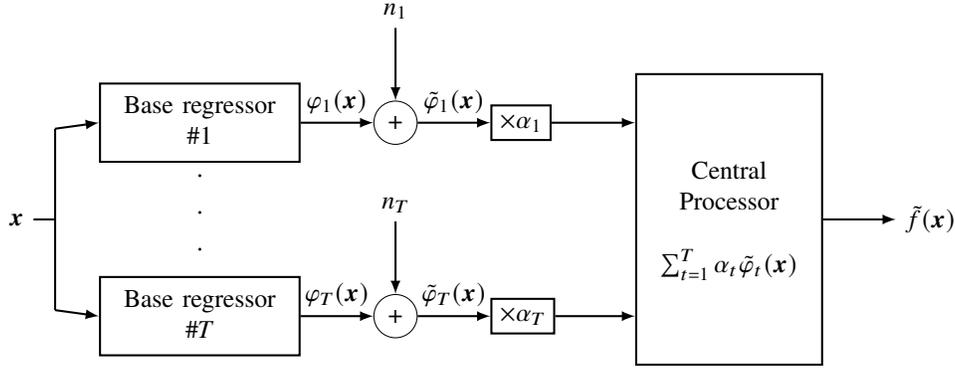

Due to their effectiveness and popularity, we specialize the discussion to \emph{linearly} aggregated ensembles $\hat{f}(\bs{z})=\bs{\alpha}^\top \bs{z}$, which in the noisy case give
\begin{equation}\label{eq:fhat-lin-comb}
    \tilde{f}(\bs{x}) = \bs{\alpha}^\top \tilde{\bs{\varphi}}(\bs{x}),
\end{equation}
where $ \tilde{\bs{\varphi}}(\bs{x}) = \left(\tilde{\varphi}_1(\bs{x}), \dots, \tilde{\varphi}_T(\bs{x})\right)^\top $ denotes a vector of channel outputs and $ \bs{\alpha} \in \mathbb{R}^T $ is an aggregation coefficient vector.
For every data sample $\bs{x}\in\mathbb{R}^u$, each prediction $\tilde{\varphi}_t(\bs{x})$ is weighted by an aggregation coefficient $\alpha_t\in\mathbb{R}$ to produce an optimized ensemble prediction.
These coefficients are determined during training (or post-training) toward optimizing the final predictions, and they do not change at inference time.

For the remainder of this paper, we address this setting of a linearly-aggregated regression ensemble with additive noise. Our proposed methods aim to optimize the noisy ensemble prediction defined next.
\begin{definition}[noisy ensemble prediction]\label{def:additive-noisy-pred}
    Let $\bs{n} = \left(n_1,\dots,n_T \right)^\top$ be a random vector.
    For a data sample $\bs{x}$, define the noisy ensemble prediction as 
    \begin{equation}\label{eq:def-additive-noisy-pred}
        \tilde{f}(\bs{x}) = \bs{\alpha}^\top \tilde{\bs{\varphi}}(\bs{x}),
        \textnormal{ where } \tilde{\bs{\varphi}}(\bs{x}) = \bs{\varphi}(\bs{x}) + \bs{n},
    \end{equation}
    $ \bs{\varphi}(\bs{x}) = \left(\varphi_1(\bs{x}), \dots, \varphi_T(\bs{x})\right)^\top $ is a vector of base-regressor outputs, and $\bs{\alpha}\in\mathbb{R}^T$ is the ensemble's coefficient vector.
\end{definition}

While a ``noiseless'' regression ensemble can be trained to minimize the loss over model realizations (using a training dataset), the noisy setting requires to consider the expected loss for noise realizations as well.
In other words, training should now seek not only low model error, but also robustness to aggregated prediction noise. 
It is therefore natural to consider the following ``doubly'' expected loss. 
\begin{definition}\label{def:expected-cost}
    Let $\tilde{f}(\cdot): \mathbb{R}^u \rightarrow \mathbb{R}$ be a noisy ensemble prediction function with a coefficient vector $\bs{\alpha}\in\mathbb{R}^T$, and $f(\cdot)$ the target function to predict. For a loss function $\mathcal{L}$, define the expected loss as
    \begin{equation}\label{eq:expected-cost-per-sample}
        \tilde{J}_{\mathcal{L}} \triangleq 
        J_{\mathcal{L}}(\tilde{f}) \triangleq 
        \mathbb{E}_{\bs{n}} \left[
        \mathcal{L}\left(
        f,\tilde{f}
        \right)
        \right]
        ,
    \end{equation}
    where the loss is calculated over evaluations of the function $f(\bs{x})$ using a set of data samples $\bs{x}\in\mathcal{X}$ and the expectation is taken over realizations of the noise $\bs{n}$.
\end{definition}

For example, the loss $\mathcal{L}$ can be the \emph{mean absolute error} (MAE) $ \tilde{J}_{1}\triangleq \mathbb{E}_{\bs{n}}\left[\frac{1}{\left|\mathcal{X}\right|}\sum_{\bs{x}\in\mathcal{X}} \left|f(\bs{x}) - \tilde{f}(\bs{x})\right|\right] $, the \emph{mean squared error} (MSE) $ \tilde{J}_{2}\triangleq \mathbb{E}_{\bs{n}}\left[\frac{1}{\left|\mathcal{X}\right|} \sum_{\bs{x}\in\mathcal{X}} \left|f(\bs{x}) - \tilde{f}(\bs{x})\right|^2 \right]$, or the $\ell_p$-norm based loss $ \tilde{J}_{\ell_p}\triangleq \mathbb{E}_{\bs{n}}\left[\sqrt[\leftroot{-2}\uproot{2}p]{ \frac{1}{\left|\mathcal{X}\right|} \sum_{\bs{x}\in\mathcal{X}} \left|f(\bs{x}) - \tilde{f}(\bs{x})\right|^p }\right] $.

The primary objective of this paper is to minimize this loss function.
Practically, the model samples $\bs{x}, f(\bs{x})$ are realized through the training data samples in $\mathcal{S}$, and $\tilde{f}(\bs{x})$ is realized through noisy observations of the model $\hat{f}(\bs{\varphi}(\bs{x}))$. 
Hence, in the forthcoming optimization results the loss is calculated using a sum over the training samples with $y$ taking the role of $f(\bs{x})$, and then taking the expectation over the noise $\bs{n}$. For testing the generalization of the optimized models, a \emph{test set} not used in training and optimization is taken as $\mathcal{X}$ (with corresponding $y$ values) in the loss expressions.
Note that the noise $\bs{n} = \tilde{\bs{\varphi}}(\bs{x})-\bs{\varphi}(\bs{x})$ is only present at inference: it is not visible during training, and in our post-training optimization only the \emph{model} of the noise is assumed known (no noise samples are used in optimization).
For comparing noisy to noiseless performance, we define the \emph{noiseless loss} $J_{\mathcal{L}}$ in a similar way to $\tilde{J}_{\mathcal{L}}$, but with $\tilde{f}(\bs{x})$ replaced by $\hat{f}(\bs{\varphi}(\bs{x}))$ and without the expectation over $\bs{n}$ (clearly, because there is no noise).

To offer a degree of separation between the machine-learning layer and the communication layer, in the first part of the paper (Section \ref{sec:robust-bagging}) we assume that the base regressors are trained without considering the noisy setting, thus mitigating the noise through the task of optimizing the aggregation coefficients post-training. Such a separation is motivated by settings in which the channels carrying predictions are not known at the time of training. 
We now state the primary problem addressed in the first part of this paper.
\begin{problem}\label{prob:general-problem}    
    Given a dataset $\mathcal{S} = \{(\bs{x}_i,y_i)\}_{i=1}^{N_s}$, an ensemble of base regressors $\{\varphi_{t}(\cdot)\}_{t=1}^{T}$ and a distribution of the noise $\bs{n}$, find the coefficient vector $\bs{\alpha} = (\alpha_1,\dots,\alpha_T)^\top$ that minimizes $\tilde{J}_{\mathcal{L}}(\tilde{f})$.
\end{problem}
Specifically, the majority of our results concentrates on the most widely accepted figures of merit in the context of regression: the MSE ($\tilde{J}_{2}$) and the MAE ($\tilde{J}_{1}$).
Also, unless explicitly stated otherwise, the noise vector is assumed to be distributed with mean $\bs{0}$ and covariance matrix $\bs{\Sigma}$.
For some specific cases in the sequel, in which we assume the noise is normally distributed, we denote the probability density function (PDF) and the cumulative density function (CDF) of a normal random variable with mean $0$ and variance $1$ as $g(t)\triangleq\frac{1}{\sqrt{2\pi}}e^{\frac{t^2}{2}}$ and $\phi(\psi)\triangleq\frac{1}{\sqrt{2\pi}}\int_{-\infty}^{\psi}e^{\frac{t^2}{2}}dt$, respectively. The error function, defined as $2\phi(\sqrt{2}x)-1$, is denoted $\textnormal{erf}(x)$ and the complementary error function $1-\textnormal{erf}(x)$ is denoted $\textnormal{erfc}(x)$.

The post-training coefficient optimization of Problem~\ref{prob:general-problem} is motivated by bagging methods for ensemble aggregation that provide ensembles of base regressors trained individually on different subsets of the training set~\cite{perrone1993improving}.

\subsection{Motivation for robust prediction}\label{sec:motivating-example}
To motivate the problem of robust regression, we exemplify the adverse effect of noise on a standard regression bagging ensemble. Consider an ensemble of $T=5$ decision-tree base regressors (with a maximal depth of $8$), where each base regressor was trained on a different subset of the training set (known as ``bagging''). The final prediction is obtained by averaging the individual predictions of the base regressors.

To conduct a consistent evaluation, we define the following signal-to-noise ratio ($\mathrm{SNR}$) measures. For a single base regressor $\varphi_t(\cdot)$, we define $\mathrm{SNR}_t$ as $\frac{\varepsilon_{y}}{\sigma^2_t}$, where $\varepsilon_{y} \triangleq \frac{1}{N_s}\sum_{i=1}^{N_s}y_i^2$ is the normalized sum of squares of the training-dataset target values, and $\sigma^2_t$ is the $t$-th channel's noise variance. 
For the entire ensemble, the $\mathrm{SNR}$ definition is extended to
\begin{equation}\label{eq:snr-def}
    \mathrm{SNR} = \frac{T \varepsilon_{y}}{\textnormal{Tr}(\bs{\Sigma})},
\end{equation}
where $\bs{\Sigma}$ is the noise covariance matrix and $T$ is the ensemble size. When given in dB, the signal-to-noise ratio is calculated as $10\log_{10}(\mathrm{SNR})$.

In this sub-section, we exemplify a scenario in which a single base regressor communicates through a low-$\mathrm{SNR}$ channel, while the remainder of the ensemble is transmitted through channels with better $\mathrm{SNR}$s. This setting is realized by a diagonal noise covariance matrix, in which the first diagonal element is set to factor $a\gg1$ larger than the remaining diagonal elements. For a given ensemble $\mathrm{SNR}$ and $a$, the variance of the first channel $t=1$ is set to
\begin{equation}
    \sigma^2_1 = \frac{T\varepsilon_y}{\left(1+\frac{T-1}{a}\right)\mathrm{SNR}}
\end{equation} 
and the remaining variances $1<t\leq T$ are $\sigma^2_t = \frac{1}{a} \sigma^2_1$.
We experiment with a synthetic dataset $y=f(x)+\epsilon$ where $\epsilon\sim\mathcal{N}(0, 0.01)$ is a measurement noise and the target function is generated by sampling a sinusoidal sum: $f(x) = \sin(x) + \sin(6x) \textnormal{, where } 0\leq x\leq6$.

We evaluate in Fig.~\ref{fig:toy-example} the (root) MSE and the MAE with and without noise whose parameters are $\mathrm{SNR}=-6$dB, $a=20$. The prediction plots, as well as the resulting average MSE and MAE values (appearing on the lower left of the figure), illustrate the corrupting effects of channel noise on the quality of the ensemble's prediction over the test set.
In addition to the ground-truth values $y$ (denoted: ``Test: Ground truth''), Fig.~\ref{fig:toy-example} shows two additional curves: 1) predictions produced without noise (denoted ``Test: Noiseless prediction''); and 2) predictions produced from the noisy base regressors (denoted ``Test: Noisy prediction'').

The poor performance of the noisy predictor is easily observable.
It can be seen that even a single base regressor with non-negligible noise can significantly degrade the overall ensemble performance for both the MSE and MAE measures.
These results motivate the mitigation of base regressors' channel noise by optimization of the aggregation coefficients.

\begin{figure}
        \centering
        \includegraphics[width=\linewidth]{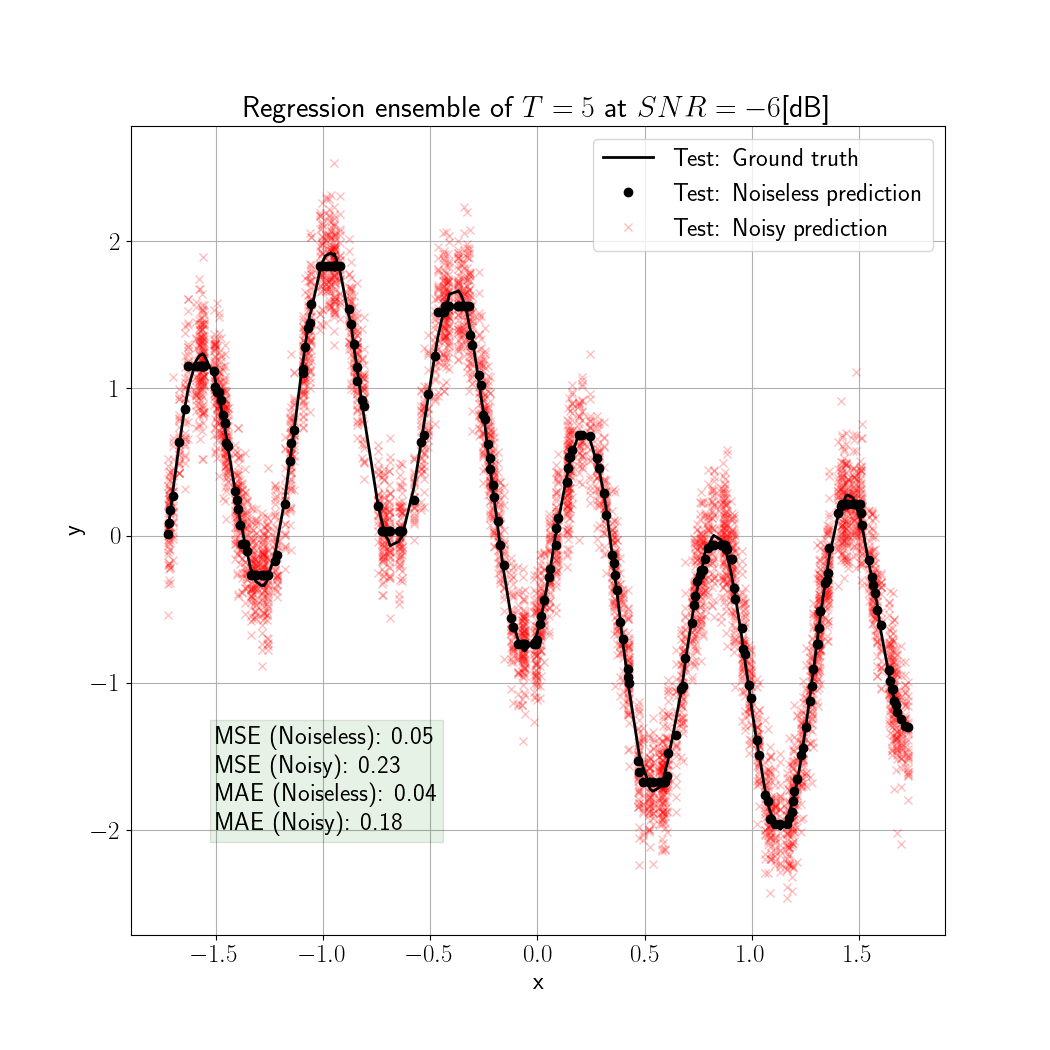}
    \caption{Illustration of noisy and noiseless prediction for a Sine target function, with the corresponding values of its (root) MSE and MAE. Note that $x$ was centralized to $0$.
    }
    \label{fig:toy-example}
\end{figure}

\section{Robust Bagging}\label{sec:robust-bagging}
Toward carrying out Problem~\ref{prob:general-problem}'s post-training optimization of the aggregation coefficient vector $\bs{\alpha}$, the following result is obtained for a general $\ell_p$-norm based loss.
\begin{theorem}\label{thm:ellp-err-upper-bnd}
    For any coefficient vector $\bs{\alpha}$ and $p\in\mathbb{N}$,
    \begin{equation}\label{eq:upper-bnd-general}
        \tilde{J}_{\ell_p}(\bs{\alpha}) \leq 
        J_{\ell_p}(\bs{\alpha}) + 
        \mathbb{E}_{\bs{n}} \left[ \bs{\alpha}^\top\bs{n} \right].
    \end{equation}
    where $\bs{n}$ is the noise random vector and $J_{\ell_p}(\bs{\alpha})$ is the $\ell_p$ loss of the noiseless predictor $\bs{\alpha}^\top \bs{\varphi}(\cdot)$.
\end{theorem}
\begin{proof}
    Using the triangle inequality to obtain an upper bound and then simplifying the resulting expressions, we get
    \begin{equation}
    \begin{aligned}
        \tilde{J}_{\ell_p}(\bs{\alpha}) 
        & = 
        \mathbb{E}_{\bs{n}} \left[ 
        \sqrt[\leftroot{-2}\uproot{2}p]{
        \frac{1}{N_s} \sum_{i=1}^{N_s} \left| y_i - \bs{\alpha}^\top\tilde{\bs{\varphi}}(\bs{x}_i) \right|^p 
        }
        \right] 
        \\ & =
        \mathbb{E}_{\bs{n}} \left[ 
        \sqrt[\leftroot{-2}\uproot{2}p]{
        \frac{1}{N_s} \sum_{i=1}^{N_s} \left| y_i - \bs{\alpha}^\top\bs{\varphi}(\bs{x}_i) - \bs{\alpha}^\top\bs{n} \right|^p 
        }
        \right]
        \\ & \leq
        \sqrt[\leftroot{-2}\uproot{2}p]{
        \frac{1}{N_s} \sum_{i=1}^{N_s} \left| y_i - \bs{\alpha}^\top\bs{\varphi}(\bs{x}_i) \right|^p 
        }
        +
        \mathbb{E}_{\bs{n}} \left[ \left| \bs{\alpha}^\top\bs{n} \right| \right]
        \\ & =
        J_{\ell_p}(\bs{\alpha})
        +
        \mathbb{E}_{\bs{n}} \left[ \left| \bs{\alpha}^\top\bs{n} \right| \right]
        .
    \end{aligned}
    \end{equation}
\end{proof}

Note that Theorem~\ref{thm:ellp-err-upper-bnd} applies to the MAE, because it coincides with the $\ell_1$ loss for $p=1$, but not to the MSE due to the square root in the $\ell_2$ norm that is absent in the MSE. The upper bound of Theorem~\ref{thm:ellp-err-upper-bnd} provides insight into the structure of prediction errors in noisy ensemble regression:~\eqref{eq:upper-bnd-general} facilitates a decomposition of the loss to two types. 
The first relates to model error, while the second is associated with the expected aggregated noise. 
In general, we seek to minimize both quantities for robust regression, but maintaining the distinction between them will be useful for specializing the optimization to particular system settings.

As the baseline approaches for ensemble aggregation, we consider the \emph{basic ensemble method} (BEM)~\cite{perrone1993improving}, according to which the coefficients are all equal
\begin{equation}
    \bs{\alpha}^{(BEM)} \triangleq \frac{1}{T}\bs{1},
\end{equation}
and the \emph{generalized ensemble method} (GEM)~\cite{perrone1993improving}, which sets $\bs{\alpha}$ to minimize $J_{\ell_p}$ among \emph{normalized} coefficients:
\begin{equation}\label{eq:gem-coeff}
    \bs{\alpha}^{(GEM)}_{\ell_p} \triangleq \arg\min_{\bs{\alpha}\in\mathbb{R}^T: \bs{1}^\top\bs{\alpha}=1} J_{\ell_p}(\bs{\alpha}).
\end{equation}
(recall that $J_{\ell_p}(\bs{\alpha})$ is the noiseless $\ell_p$ loss.) Note that GEM can be considered as a non-robust counterpart of our MSE-optimal approach presented now.
 



\subsection{Robust bagging for MSE loss}\label{subsec:bagging-mse}

We start by pursuing aggregation coefficients that optimize the mean squared error (MSE) loss. The MSE loss expression can be decomposed into two sub-terms similar to~\eqref{eq:upper-bnd-general}, but with equality instead of as an upper bound. 
\begin{proposition}\label{prop:mse-tilde-f}
    Let $\bs{\Sigma}$ be the noise covariance matrix, and let $J_2(\bs{\alpha})$ be the noiseless MSE loss of a regression ensemble aggregated with coefficients $\bs{\alpha}\in\mathbb{R}^T$.
    Then, for any $\bs{\alpha}\in\mathbb{R}^T$,
    \begin{equation}\label{eq:expected-mse-general-expression}
        \tilde{J}_{2}(\bs{\alpha}) = J_{2}(\bs{\alpha}) + \bs{\alpha}^\top \bs{\Sigma} \bs{\alpha}.
    \end{equation}
\end{proposition}
\begin{proof}
    Recalling that $\mathbb{E}\left[\bs{n}\right] = \bs{0}$, the loss decomposes as follows
    \begin{equation}\label{eq:mse-proof-eq-2}
    \begin{aligned}
        \tilde{J}_{2}(\bs{\alpha}) & = 
        \mathbb{E}_{\bs{n}} \left[ \frac{1}{N_s}\sum_{i=1}^{N_s} \left|y_i-\tilde{f}(\bs{x}_{i})\right|^2 \right]
        \\ & =
        \mathbb{E}_{\bs{n}} \left[ \frac{1}{N_s}\sum_{i=1}^{N_s} \left| y_i - \bs{\alpha}^\top\bs{\varphi}(\bs{x}_{i}) - \bs{\alpha}^\top\bs{n} \right|^2 \right]
        \\ & =
        \frac{1}{N_s}\sum_{i=1}^{N_s} \mathbb{E}_{\bs{n}} \Big[ \left( y_i - \bs{\alpha}^\top\bs{\varphi}(\bs{x}_{i}) \right)^2 + \left( \bs{\alpha}^\top\bs{n} \right)^2 \\ & - 2 \left( y_i - \bs{\alpha}^\top\bs{\varphi}(\bs{x}_{i}) \right) \left( \bs{\alpha}^\top\bs{n} \right) \Big]
        \\ & =
        \frac{1}{N_s}\sum_{i=1}^{N_s} \left( y_i - \bs{\alpha}^\top\bs{\varphi}(\bs{x}_{i}) \right)^2 + \mathbb{E}_{\bs{n}} \left[ \left( \bs{\alpha}^\top\bs{n} \right)^2 \right]
        .
    \end{aligned}
    \end{equation}
    The first term equals $ J_{2}(\bs{\alpha})$ by definition, and the second term equals
    \begin{equation*}
        \bs{\alpha}^\top \mathbb{E}_{\bs{n}} \left[ \bs{n}\bs{n}^\top \right] \bs{\alpha}
        =
        \bs{\alpha}^\top \bs{\Sigma} \bs{\alpha}.
    \end{equation*}
\end{proof}

Theorem~\ref{thm:ellp-err-upper-bnd} and Proposition~\ref{prop:mse-tilde-f} manifest an inherent trade-off between model error (the first terms in~\eqref{eq:upper-bnd-general}, \eqref{eq:expected-mse-general-expression}) and aggregated noise (the second terms). This motivates the refinement of Problem~\ref{prob:general-problem} to allow tuning of the contribution of each loss component.
By introducing a coefficient $\lambda\geq0$ that weights the aggregated-noise component, we get the following problem that optimizes a more general variant of the MSE loss.
\begin{customproblem}{\ref{prob:general-problem}'}\label{prob:l2-bagging}
    Given a dataset $\mathcal{S} = \{(\bs{x}_i,y_i)\}_{i=1}^{N_s}$, an ensemble of base regressors $\{\varphi_{t}(\cdot)\}_{t=1}^{T}$ and a zero-mean noise vector with covariance matrix $\bs{\Sigma}$,
    find the coefficients $\bs{\alpha} = (\alpha_1,\dots,\alpha_T)^\top$ that minimize 
    \begin{equation}\label{eq:l2-bagging}
        \tilde{J}_{2}^{(\lambda)}(\bs{\alpha}) = J_2\left(\bs{\alpha}\right) + \lambda\cdot\bs{\alpha}^\top \bs{\Sigma} \bs{\alpha}.
    \end{equation}
\end{customproblem}
We refer to solving Problem~\ref{prob:l2-bagging} as the \emph{trade-off ensemble method} (TEM).
Note that \eqref{eq:l2-bagging} may be cast back to Problem~\ref{prob:general-problem} (for MSE loss) by setting $\lambda=1$. In practice, \emph{bounding} the aggregated-noise component may be more convenient than weighting it. We therefore re-formulate the TEM to solve the following equivalent constrained optimization problem
\begin{equation}\label{eq:l2-bagging-cons-opt}
    \min_{\bs{\alpha}\in\mathbb{R}^T} 
    J_2(\bs{\alpha})
    \textnormal{ subject to } 
    \bs{\alpha}^\top \bs{\Sigma} \bs{\alpha} \leq C.
\end{equation}
\eqref{eq:l2-bagging-cons-opt} considers the minimization of the ``standard'' (noiseless) model error while bounding the contribution of the aggregated noise to the overall loss term.
This formulation gives precedence to minimizing the model-error term, while ensuring that the error from noise is kept sufficiently small (e.g., within the noise tolerance of the regression task). The advantage of this approach is that making the noise a secondary optimization criterion leads to better robustness to imperfect knowledge of the noise parameters. We now proceed to solving this problem. The following characterizes the solution of~\eqref{eq:l2-bagging-cons-opt} as a least-squares problem with a quadratic constraint (see~\cite{gander1980least}); afterwards we detail a method to solve the problem efficiently using an approximation.

Towards deriving the solution of~\eqref{eq:l2-bagging-cons-opt}, we define the prediction matrix of dimensions $N_s \times T$.
\begin{definition}\label{def:pred-mtx}
    Given an ensemble of base regressors $\left\{\varphi_t(\cdot)\right\}_{t=1}^{T}$ and a set of inputs $\left\{\bs{x}_i\right\}_{i=1}^{N_s}$, define the prediction matrix as
    \begin{equation}\label{eq:pred-mtx}
        \bs{\Phi} \triangleq \left[ \bs{\varphi}(\bs{x}_1), \dots, \bs{\varphi}(\bs{x}_{N_s}) \right]^\top.
    \end{equation}
\end{definition}

\begin{theorem}\label{thm:tem-mse}
    Let $\bs{\Sigma}$ be the noise covariance matrix and let $\bs{\Phi}$ be a rank-$T$ prediction matrix of the regression ensemble.
    The optimal solution of the optimization problem in~\eqref{eq:l2-bagging-cons-opt} is
    \begin{equation}\label{eq:tem-coeff}
        \bs{\alpha}^* = \left( \bs{\Phi}^\top\bs{\Phi} + \lambda N_s \bs{\Sigma} \right)^{-1} \bs{\Phi}^\top \bs{y},
    \end{equation}
    where $\bs{y} \triangleq (y_1,\dots,y_{N_s})^\top$ and $\lambda$ is either $0$ or given by the solution of
    \begin{equation}\label{eq:lambda-equation}
        \bs{y}^\top \bs{\Phi} \left( \bs{\Phi}^\top\bs{\Phi}+\lambda N_s \bs{\Sigma} \right)^{-\top} \bs{\Sigma} \left( \bs{\Phi}^\top\bs{\Phi}+\lambda N_s \bs{\Sigma} \right)^{-1}  \bs{\Phi}^\top \bs{y} = C.
    \end{equation}
\end{theorem}
\begin{proof}
    This proof is similar to a proof given for a related optimization problem in~\cite{gander1980least}. Since the minimized function $J_2(\bs{\alpha}) = \frac{1}{N_s}\|\bs{\Phi}\bs{\alpha}-\bs{y}\|_2^2$ and the inequality constraint function $\bs{\alpha}^\top\bs{\Sigma}\bs{\alpha} - C$ are both convex (as quadratic forms), we can use the Karush-Kuhn-Tucker (KKT) conditions to find the optimal $\bs{\alpha}$. Since $\nabla J_2(\bs{\alpha}) = \frac{2}{N_s} \bs{\Phi}^\top (\bs{\Phi}\bs{\alpha}-\bs{y})$ and $\nabla \left( \bs{\alpha}^\top\bs{\Sigma}\bs{\alpha} - C \right) = 2\bs{\Sigma}\bs{\alpha}$, the KKT conditions are
    \begin{equation}
    \begin{aligned}
        (\bs{\alpha}^\ast)^\top\bs{\Sigma}\bs{\alpha}^\ast - C & \leq 0
        \\ \lambda & \geq 0
        \\ \lambda \left( (\bs{\alpha}^\ast)^\top\bs{\Sigma}\bs{\alpha}^\ast - C \right) & = 0
        \\ \frac{2}{N_s} \bs{\Phi}^\top (\bs{\Phi}\bs{\alpha}^\ast-\bs{y}) + 2 \lambda \bs{\Sigma}\bs{\alpha}^\ast & = \bs{0}
    \end{aligned}
    \end{equation}
    According to the last condition, $ \bs{\alpha}^\ast = (\bs{\Phi}^\top\bs{\Phi} + \lambda N_s \bs{\Sigma})^{-1} \bs{\Phi}^\top \bs{y} $.
    Now, from the third condition either $\lambda=0$ or $(\bs{\alpha}^\ast)^\top\bs{\Sigma}\bs{\alpha}^\ast = C$. Hence, whenever the pseudo-inverse solution $\bs{\alpha}^* = \left( \bs{\Phi}^\top\bs{\Phi} \right)^{-1} \bs{\Phi}^\top \bs{y}$ satisfies the constraint $(\bs{\alpha}^\ast)^\top\bs{\Sigma}\bs{\alpha}^\ast \leq C$, set $\lambda = 0$.
    When the constraint is not satisfied by the pseudo-inverse solution, it follows that $(\bs{\alpha}^\ast)^\top\bs{\Sigma}\bs{\alpha}^\ast = C$, which gives~\eqref{eq:lambda-equation} when $\bs{\alpha}^\ast$ is substituted.
\end{proof}

Theorem~\ref{thm:tem-mse} provides a closed-form expression for the optimal coefficients solving~\eqref{eq:l2-bagging-cons-opt} when $\lambda$ is known, in particular for $\lambda=1$ corresponding to the standard MSE. However, given $C$ the solution depends on $\lambda$, which is given only implicitly by~\eqref{eq:lambda-equation}. Although $\lambda$ can be calculated by complex numerical methods, we now suggest an alternative approach that allows to calculate $\lambda$ in closed form after approximating the matrix inverse in~\eqref{eq:lambda-equation} using the Neumann series.

\begin{lemma}\label{lemma:tem-mse-approx}
Let $\lambda$ be a solution of the equation
 \begin{equation}\label{eq:approx_lambda-equation}
        \bs{y}^\top \bs{\Phi} Z^{\top} \bs{\Sigma} Z  \bs{\Phi}^\top \bs{y} = C,
    \end{equation}
    where $Z\triangleq \left( \bs{I} + \lambda \left( \bs{\Phi}^\top \bs{\Phi} \right)^{-1} \bs{\Sigma} \right) \left( \bs{\Phi}^\top \bs{\Phi} \right)^{-1}$. Then $\lambda$ can be found as a root of the second-order polynomial
    \begin{equation}\label{eq:tem-mse-approx-lambda-poly}
        a \lambda^2 + b \lambda + c,
    \end{equation}
    where the coefficients are
    \begin{equation}\label{eq:tem-mse-approx-lambda-poly-coeffs}
    \begin{aligned}
        a & = \bs{y}^\top \bs{\Phi} \bs{\Sigma}^\top (\bs{\Phi}^\top\bs{\Phi})^{-\top} \bs{\Sigma} (\bs{\Phi}^\top\bs{\Phi})^{-1} \bs{\Sigma} \bs{\Phi}^\top \bs{y}
        \\ b & = \bs{y}^\top \bs{\Phi} \bs{\Sigma}^\top \left(  (\bs{\Phi}^\top\bs{\Phi})^{-\top} + (\bs{\Phi}^\top\bs{\Phi})^{-1} \right) \bs{\Sigma} \bs{\Phi}^\top \bs{y}
        \\ c & = \bs{y}^\top \bs{\Phi} \bs{\Sigma} \bs{\Phi}^\top \bs{y} - C
        .
    \end{aligned}
    \end{equation}
\end{lemma}

We omit the proof of Lemma~\ref{lemma:tem-mse-approx} since it can be verified by expanding~\eqref{eq:approx_lambda-equation} followed by straightforward rearrangements.
Since $Z$ in Lemma~\ref{lemma:tem-mse-approx} is a well-known approximation of $ \left( \bs{\Phi}^\top\bs{\Phi}+\lambda\bs{\Sigma} \right)^{-1}$ (see~\cite{seber2008matrix}, Theorems 19.16(a), (b), (c)), finding $\lambda$ by solving the quadratic polynomial~\eqref{eq:tem-mse-approx-lambda-poly} will give an approximation of the optimal value.


\subsection{Robust bagging for MAE loss}\label{subsec:bagging-mae}
In this subsection, we move to consider the mean absolute error (MAE) loss in the noisy regression problem.
The expected MAE loss with respect to a certain dataset $\{\left(\bs{x}_i,y_i\right)\}_{i=1}^{N_s}$ is given by 
\begin{equation}
\begin{aligned}
    \tilde{J}_{1}(\bs{\alpha}) 
    & = 
    \mathbb{E}_{\bs{n}}\left[ \frac{1}{N_s}\sum_{i=1}^{N_s} \left| \tilde{f}(\bs{x}_i) - y_i \right| \right]
    \\ & =
    \mathbb{E}_{\bs{n}}\left[ \frac{1}{N_s}\sum_{i=1}^{N_s} \vert\bs{\alpha}^\top \bs{\varphi}(\bs{x}_i) - y_i + \bs{\alpha}^\top\bs{n}\vert \right].
\end{aligned}
\end{equation}
Note that specifically for $p=1$ we have $\tilde{J}_{1}(\bs{\alpha}) = \tilde{J}_{\ell_1}(\bs{\alpha})$ so the upper bound from Theorem~\ref{thm:ellp-err-upper-bnd} holds.
For $p=1$ the loss does not decompose with equality like the MSE in~\eqref{eq:expected-mse-general-expression}, but for normally distributed noise it can still be written in closed form.
\begin{theorem}\label{thm:l1-loss-expression}
    Let $\bs{\alpha}\in\mathbb{R}^T$ be a coefficient vector and let the noise random vector be $\bs{n}\sim\mathcal{N}(\bs{0},\bs{\Sigma})$.
    The expected MAE loss is
    \begin{equation}\label{eq:MAE-explicit}
    \begin{aligned}
        \tilde{J}_{1}(\bs{\alpha}) = 
        \frac{1}{N_s} & \sum_{i=1}^{N_s}
        \left[
        \sqrt{\frac{2}{\pi}} \sigma(\bs{\alpha}) \exp\left(-\frac{\mu^2_{i}(\bs{\alpha})}{2\sigma^2(\bs{\alpha})}\right)\right.
        \\ & \left. +
        \mu_{i}(\bs{\alpha}) \left( 2\phi\left(\frac{\mu_{i}(\bs{\alpha})}{\sigma(\bs{\alpha})}\right)-1 \right)
        \right]
        ,        
    \end{aligned}
    \end{equation}
    where $\mu_{i}(\bs{\alpha}) \triangleq \bs{\alpha}^\top \bs{\varphi}(\bs{x}_i) - y_i$, $\sigma(\bs{\alpha}) \triangleq \sqrt{\bs{\alpha}^\top\bs{\Sigma}\bs{\alpha}}$ and $\phi(t)$ is the standard normal cumulative distribution function.
\end{theorem}
\begin{proof}
    Recall that
    \begin{equation}\label{eq:l1-loss-expression-proof-1}
    \begin{aligned}
        \tilde{J}_{1}(\bs{\alpha}) & = \mathbb{E}_{\bs{n}}\left[\frac{1}{N_s} \sum_{i=1}^{N_s} \vert \bs{\alpha}^\top\bs{n} + \bs{\alpha}^\top \varphi(\bs{x}_i) - y_i \vert\right]
        \\ & =
        \frac{1}{N_s} \sum_{i=1}^{N_s} \mathbb{E}_{\bs{n}}\left[\vert \bs{\alpha}^\top\bs{n} + \bs{\alpha}^\top \varphi(\bs{x}_i) - y_i \vert\right]
        ,
    \end{aligned}
    \end{equation} 
    where $\bs{n}$ is~$\mathcal{N}(\bs{0},\bs{\Sigma})$. Hence, the expression inside the absolute value in~\eqref{eq:l1-loss-expression-proof-1} is a normal random vector with expectation $\mu_{i}(\bs{\alpha}) = \bs{\alpha}^\top\bs{\varphi}(\bs{x}_i) - y_i$ and variance $\sigma^2(\bs{\alpha})=\bs{\alpha}^\top\bs{\Sigma}\bs{\alpha}$.
    Therefore, the argument of the expectation in~\eqref{eq:l1-loss-expression-proof-1} is a \emph{folded normal} random variable, whose expectation is known~\cite{leone1961folded} to be the expression in the summation argument in~\eqref{eq:MAE-explicit}.
\end{proof}
Since $\tilde{J}_{1}(\bs{\alpha})$ does not facilitate a closed-form solution for the optimal $\bs{\alpha}$, we propose to optimize $\bs{\alpha}$ iteratively via gradient-based methods. Toward that, we now calculate the gradient of $\tilde{J}_{1}(\bs{\alpha})$ with respect to $\bs{\alpha}$.
For compactness of expressions, we denote the scalar $\rho_{i}(\bs{\alpha}) \triangleq \frac{\mu_{i}(\bs{\alpha})}{\sigma(\bs{\alpha})}$, and its gradient with respect to $\bs{\alpha}$ is denoted $\bs{\rho}'_{i}(\bs{\alpha})$. The gradients of $\mu_{i}(\bs{\alpha})$ and $\sigma(\bs{\alpha})$ with respect to $\bs{\alpha}$ are denoted $\bs{\mu}'_{i}(\bs{\alpha})$ and $\bs{\sigma}'(\bs{\alpha})$, respectively;
the expressions for these gradients are derived as $ \bs{\mu}'_{i}(\bs{\alpha}) = \bs{\varphi}(\bs{x}_i)$,
\begin{equation}\label{eq:mae-bagg-gradient-eq-2}
\begin{aligned}
    \bs{\sigma}'(\bs{\alpha}) = \frac{\bs{\Sigma} \bs{\alpha}}{\sqrt{\bs{\alpha}^\top \bs{\Sigma} \bs{\alpha}}}
    \textnormal{, }
    \bs{\rho}'_{i}(\bs{\alpha}) & = \frac{ \sigma(\bs{\alpha}) \bs{\varphi}(\bs{x}_i)- \mu_{i}(\bs{\alpha}) \bs{\sigma}'(\bs{\alpha})}{\sigma^2(\bs{\alpha})}
    .
\end{aligned}
\end{equation}
Recall that the probability density function and the cumulative distribution function of the standard normal random variable are denoted $g(\cdot)$ and $\phi(\cdot)$, respectively. Then,

\begin{proposition}\label{prop:mae-bagg-gradient}
    The gradient of $\tilde{J}_{1}(\bs{\alpha})$ from Theorem~\ref{thm:l1-loss-expression} with respect to $\bs{\alpha}$\ is
    \begin{equation}\label{eq:mae-bagg-gradient}
    \begin{aligned}
        \frac{1}{N_s}\sum_{i=1}^{N_s}
        & \Big[ \sqrt{\frac{2}{\pi}} \exp\left(-\frac{\rho^2_{i}(\bs{\alpha})}{2}\right)
        \left(
        \bs{\sigma}'(\bs{\alpha})
        -
        \sigma(\bs{\alpha}) \rho_{i}(\bs{\alpha}) \bs{\rho}'_{i}(\bs{\alpha})
        \right)
        \\ & \hspace{-.3in}+
        \bs{\mu}'_{i}(\bs{\alpha})
        \left[2\phi\left(\rho_{i}(\bs{\alpha})\right)-1\right]
        +
        2 \mu_{i}(\bs{\alpha}) \bs{\rho}'_{i}(\bs{\alpha}) g\left(\rho_{i}(\bs{\alpha})\right)
        \Big]
        .
    \end{aligned}
    \end{equation}
\end{proposition}

\begin{proof}
    The expression for $\tilde{J}_1(\bs{\alpha})$ in~\eqref{eq:MAE-explicit} comprises a sum over two terms. The gradient of the first term is (omitting the constant $\sqrt{2/\pi}$) 
    \begin{equation}
        \exp\left(-\frac{\rho^2_{i}(\bs{\alpha})}{2}\right)\bs{\sigma}'(\bs{\alpha})
        -
        \sigma(\bs{\alpha})
        \exp\left(-\frac{\rho^2_{i}(\bs{\alpha})}{2}\right) \rho_{i}(\bs{\alpha}) \bs{\rho}'_{i}(\bs{\alpha}).
    \end{equation}
    The gradient of the second term is 
    \begin{equation}
        \bs{\mu}'_{i}(\bs{\alpha})\left[1 - 2\phi\left(-\rho_{i}(\bs{\alpha})\right)\right]
           +
        2 \mu_{i}(\bs{\alpha})  g\left(-\rho_{i}(\bs{\alpha})\right)\bs{\rho}'_{i}(\bs{\alpha}).
    \end{equation}
    Note that we used the fact that the derivative of $\phi(\cdot)$ is the normal density function.
    Finally, we use $g(t) = g(-t)$ and $\phi(-t) = 1 - \phi(t)$ to obtain~\eqref{eq:mae-bagg-gradient}.
\end{proof}

The following Algorithm~\ref{algo:grad-dec-bagg} gives a gradient-descent procedure to minimize $\tilde{J}_p(\bs{\alpha})$. Substituting~\eqref{eq:mae-bagg-gradient} from Proposition~\ref{prop:mae-bagg-gradient} as $\nabla \tilde{J}_1(\bs{\alpha})$ in line 9 allows using the algorithm for the MAE ($p=1$), which is done in the experimental study presented in Section~\ref{sec:experiments}. To allow comparison with non-robust aggregation, we use as baseline Algorithm~\ref{algo:grad-dec-bagg} with the noise-free loss $J_1(\bs{\alpha})$ replacing $\tilde{J}_1(\bs{\alpha})$, resulting in the noise-free gradient 
\begin{equation}\label{eq:non-robust-bagging-mae-grad}
    \nabla J_1(\bs{\alpha}) = \frac{1}{N_s}\sum_{i=1}^{N_s} \bs{\varphi}\left(\bs{x}_i\right) \sign\left( \bs{\alpha}^\top \bs{\varphi}\left(\bs{x}_i\right) - y_i \right).
\end{equation}

\begin{algorithm}
    \caption{Gradient-descent minimization of $\tilde{J}_p(\bs{\alpha})$}
    \label{algo:grad-dec-bagg}
	\begin{algorithmic}[1]
        \STATE \underline{Input}:$ \{(\bs{x}_i,y_i)\}_{i=1}^{N_s} $, $ \{\varphi_t(\cdot)\}_{t=1}^{T} $, $\bs{\Sigma}$ - training dataset, trained base classifiers and noise covariance matrix
        \STATE \underline{Output}: $ \{\alpha_t\}_{t=1}^{T} $ - optimized aggregation coefficients.
		\STATE Set: $i_{\min}, i_{\max}$ (\# iter.), $\eta$ (learn rate), $\gamma$ (momentum), $\tau$ (tolerance) and $\epsilon$ (regularization)
		\STATE Initialize: $\bs{\alpha}^{(1)} \leftarrow \frac{1}{T}\bs{1}$ , $\delta_{\bs{\alpha}}^{(-1)} \leftarrow \bs{0}$ and $i \leftarrow 0$ 
		\WHILE{$i \leq i_{\max}$}
		    \IF {$ |\tilde{J}_p(\bs{\alpha}^{(i)}) - \tilde{J}_p(\bs{\alpha}^{(i-1)})| \leq \tau $ and $i \geq i_{\min}$}
                \STATE break
            \ELSE 
    		    \STATE $\delta_{\bs{\alpha}}^{(i)} \leftarrow \gamma\cdot \delta_{\bs{\alpha}}^{(i-1)} - \eta \cdot \frac{\nabla  \tilde{J}_p(\bs{\alpha}^{(i-1)})} {\sqrt{\sum_{j \leq i} \left(\nabla  \tilde{J}_p(\bs{\alpha}^{(j)})\right)^2+\epsilon}}$
    		    \STATE $\bs{\alpha}^{(i+1)} \leftarrow \bs{\alpha}^{(i)} + \delta_{\bs{\alpha}}^{(i)}$
                \STATE $i \leftarrow i+1$
            \ENDIF
   	    \ENDWHILE
	    \STATE $ i^* \leftarrow \arg\min_{0 \leq j \leq i} \tilde{J}_p(\bs{\alpha}^{(j)}) $
    	\RETURN $ \bs{\alpha}^{(i^*)} $
	\end{algorithmic}
\end{algorithm}

\subsection{Performance bounds for robust MAE}\label{subsec:mae-bounds}
Algorithm~\ref{algo:grad-dec-bagg} is useful for optimizing the MAE in practice, but we now derive performance bounds on the MAE towards gaining insight on this optimization. These bounds are also a useful tool for evaluating the expected loss without actually optimizing it. 
We derive lower and upper bounds on the optimal MAE loss $\tilde{J}_1(\bs{\alpha}^*)$, where $\bs{\alpha}^{*}$ is a coefficient vector that minimizes the MAE.

Specifically, the upper bound depends on the performance of \emph{noiseless} regression and the noise covariance matrix, and does not require to perform calculations that involve the entire dataset.
In the following Proposition~\ref{prop:MAE_ub}, we bound the MAE loss of noisy regression by the sum of the noiseless regression's MAE and a noise-related term.
\begin{proposition}\label{prop:MAE_ub} 
    Let $\bs{\alpha}^*\in\mathbb{R}^T$ be a coefficient vector that minimizes $\tilde{J}_1(\bs{\alpha})$ from Theorem~\ref{thm:l1-loss-expression} and let $\bs{\Sigma}$ be a positive-definite covariance matrix of the normally distributed noise vector.
    Then, for any $\bs{\alpha}\in\mathbb{R}^T$,
    \begin{equation}\label{eq:l1-upper-bnd}
        \tilde{J}_1(\bs{\alpha}^*) \leq 
        J_1\left(\bs{\alpha}\right)
        + 
        \sqrt{\frac{2}{\pi}\bs{\alpha}^\top\bs{\Sigma}\bs{\alpha}}.
    \end{equation}
\end{proposition}
\begin{proof}
    Since $\bs{\alpha}^*$ minimizes $\tilde{J}_1(\cdot)$, we have the first inequality in
    \begin{equation}
        \tilde{J}_1(\bs{\alpha}^*) \leq
        \tilde{J}_1(\bs{\alpha}) \leq
        J_1(\bs{\alpha}) + \mathbb{E}\left[\left|\bs{\alpha}^\top\bs{n}\right|\right]
        ,
    \end{equation}
    and the second inequality is due to the bound in Theorem~\ref{thm:ellp-err-upper-bnd}.
    Since $\bs{\alpha}^\top\bs{n}$ is a normal random variable with mean $0$ and variance $\bs{\alpha}^\top\bs{\Sigma}\bs{\alpha}$, then from properties of folded normal random variables,
    \begin{equation}
        \mathbb{E}_{\bs{n}}\left[\left|\bs{\alpha}^\top\bs{n}\right|\right]
        =
        \sqrt{\frac{2}{\pi}\bs{\alpha}^\top\bs{\Sigma}\bs{\alpha}}
        .
    \end{equation}
\end{proof}

The implication of Proposition~\ref{prop:MAE_ub} is that any $\bs{\alpha}\in\mathbb{R}^T$ that gives low loss in the noiseless setting can be used with noise, with its degradation bounded by the second term of~\eqref{eq:l1-upper-bnd}. The bound of Proposition~\ref{prop:MAE_ub} is specialized below to two particular choices of $\bs{\alpha}$: 1) the BEM vector~\cite{perrone1993improving}, and 2) the vector that minimizes the second term of~\eqref{eq:l1-upper-bnd}.

For BEM we have $\bs{\alpha}=\frac{1}{T}\bs{1}$, hence
\begin{equation}\label{eq:mae-bem-upper-bnd}
\tilde{J}_1(\bs{\alpha}^*) \leq 
    J_1\left(\frac{1}{T}\bs{1}\right)
    + 
    \frac{1}{T} \sqrt{\frac{2}{\pi}S(\bs{\Sigma})},
\end{equation}
where $S(\bs{\Sigma})$ denotes the sum of elements in $\bs{\Sigma}$.
A possibly tighter upper bound is obtained by selecting $\bs{\alpha}$ to be the normalized eigenvector of $\bs{\Sigma}$ that corresponds to the minimal eigenvalue. In this case, 
\begin{equation}\label{eq:mae-bagg-norm-upper-bnd}
    \tilde{J}_1(\bs{\alpha}^*) \leq 
    J_1\left( \frac{\bs{v}^{(\bs{\Sigma})}}{\bs{1}^\top\bs{v}^{(\bs{\Sigma})}}  \right) 
    + 
    \sqrt{\frac{2}{\pi} \sigma^2_{\min}},
\end{equation}
where $\sigma^2_{\min}$ is the minimal generalized eigenvalue of the pair $ \bs{\Sigma} $ and $\bs{1}\bs{1}^\top$, and $\bs{v}^{(\bs{\Sigma})}$ is the corresponding eigenvector (i.e., $\sigma^2_{\min}$ is the minimal real number $\omega$ and $\bs{v}^{(\bs{\Sigma})}$ is a vector $\bs{v}$ that together solve the equation $\bs{\Sigma}\bs{v} = \omega\bs{1}\bs{1}^\top\bs{v}$). The motivation to consider normalized $\bs{\alpha}$ vectors in~\eqref{eq:mae-bem-upper-bnd},~\eqref{eq:mae-bagg-norm-upper-bnd} is that normalization guarantees that if the individual base regressors are unbiased estimates of $f(\cdot)$, so is the aggregated ensemble regressor.

To complement the performance analysis, we also give a lower bound on the MAE. We use the normalization assumption, $\bs{1}^\top\bs{\alpha}=1$, to obtain a universal lower bound (independent of $\bs{\alpha}$) in Proposition~\ref{thm:mae-bagging-low-bnd-norm-coeffs}.

\begin{proposition}\label{thm:mae-bagging-low-bnd-norm-coeffs}
Let $\bs{\Sigma}$ be a positive-definite covariance matrix of the normally distributed noise vector, and let $\bs{\alpha}^{\ast}$ be a normalized coefficient vector with minimal MAE loss $\tilde{J}_1(\cdot)$. Then, 
    \begin{equation}\label{eq:mae-bagging-low-bnd-norm-coeffs}
    \begin{aligned}
        \tilde{J}_{1}(\bs{\alpha}^\ast)
        & \geq
        J_1(\bs{\alpha}^\dagger)
        +
        \frac{1}{N_s} \sum_{i=1}^{N_s} \bar{\Delta}_i \exp\left(- \sign'\left(\bar{\Delta}_i\right) \frac{\bar{\mu}^2_{i}}{2\bar{\sigma}^2}\right)
    \end{aligned}
    \end{equation}
    where $ \bs{\alpha}^\dagger $ is the optimal noiseless coefficient vector, $ \bar{\mu}_{i} \triangleq \max_{t=1,\dots,T} \left\{ \left|  \varphi_t(\bs{x}_i) - y_i \right| \right\} $, $\bar{\sigma}$ is the minimal generalized eigenvalue of the pair $\bs{\Sigma}$ and $\bs{1}\bs{1}^\top$, $\bar{\Delta}_i \triangleq \sqrt{\nicefrac{2}{\pi}}\bar{\sigma}-\bar{\mu}_i$ and $\sign'\left(z\right)=1$ for $z\geq0$ and $0$ otherwise.
\end{proposition}

\begin{proof}
    We start by modifying the MAE using some well-known identities related to the error function, $\textnormal{erf}(z)$. We use compressed notation: $\mu_{i}$ for $\mu_{i}(\bs{\alpha}^\ast)$ and $\sigma$ for $\sigma(\bs{\alpha}^\ast)$.
    \begin{equation}\label{eq:mae-lb-proof-1}
    \begin{aligned}
        & \tilde{J}_{1} (\bs{\alpha}^\ast) 
        \\ & = 
        \frac{1}{N_s} \sum_{i=1}^{N_s}
        \left[
        \sqrt{\frac{2}{\pi}} \sigma \exp\left(-\frac{\mu^2_{i}}{2\sigma^2}\right)
        +
        \mu_{i} \left(1 - 2\phi\left(-\frac{\mu_{i}}{\sigma}\right)\right)
        \right]
        \\ & =
        \frac{\sigma}{N_s} \sum_{i=1}^{N_s}\left[
        \sqrt{\frac{2}{\pi}} \exp\left(-\frac{\mu^2_{i}}{2\sigma^2}\right)
        -
        \frac{\mu_{i}}{\sigma} \textnormal{erf}\left(-\frac{\mu_{i}}{\sqrt{2}\sigma}\right)
        \right]
        \\ & =
        \frac{\sigma}{N_s} \sum_{i=1}^{N_s} \left[
        \sqrt{\frac{2}{\pi}} \exp\left(-\frac{\mu^2_{i}}{2\sigma^2}\right)
        +
        \frac{\mu_{i}}{\sigma} \textnormal{erf}\left(\frac{\mu_{i}}{\sqrt{2}\sigma}\right)
        \right]
        \\ & =
        \frac{\sigma}{N_s} \sum_{i=1}^{N_s} \left[
        \sqrt{\frac{2}{\pi}} \exp\left(-\frac{\mu^2_{i}}{2\sigma^2}\right)
        +
        \frac{\left|\mu_{i}\right|}{\sigma} \textnormal{erf}\left(\frac{\left|\mu_{i}\right|}{\sqrt{2}\sigma}\right)
        \right]
        \\ & =
        \frac{\sigma}{N_s} \sum_{i=1}^{N_s} \left[
        \sqrt{\frac{2}{\pi}} \exp\left(-\frac{\mu^2_{i}}{2\sigma^2}\right)
        +
        \frac{\left|\mu_{i}\right|}{\sigma} \left(1 - \textnormal{erfc}\left(\frac{\left|\mu_{i}\right|}{\sqrt{2}\sigma}\right)\right)
        \right].
    \end{aligned}
    \end{equation}
    The first equality is from~\eqref{eq:MAE-explicit}, and the remaining transitions are due to the following respective identities: $1-2\phi(z) = \textnormal{erf}(\nicefrac{z}{\sqrt{2}})$, $\textnormal{erf}(z) = -\textnormal{erf}(-z)$, $z\cdot\textnormal{erf}(z) = \left|z\right|\cdot\textnormal{erf}(\left|z\right|)$ and finally $\textnormal{erfc}(z) \triangleq 1-\textnormal{erf}(z)$.
    
    The last expression in the previous equation can be lower bounded using $\textnormal{erfc}(z) \leq \exp\left(-z^2\right)$ for $z>0$, giving
    \begin{equation}\label{eq:mae-lb-proof-2}
    \begin{aligned}
        & \tilde{J}_{1}(\bs{\alpha}^\ast)
        \\ & \geq
        \frac{\sigma}{N_s} \sum_{i=1}^{N_s} \left[
        \sqrt{\frac{2}{\pi}} \exp\left(-\frac{\mu^2_{i}}{2\sigma^2}\right)
        +
        \frac{\left|\mu_{i}\right|}{\sigma} \left(1 - \exp\left(-\frac{\left|\mu_{i}\right|^2}{2\sigma^2}\right)\right) \right]
        \\ & =
        \frac{1}{N_s} \sum_{i=1}^{N_s}
        \left(
        \sqrt{\frac{2}{\pi}}\sigma - \left|\mu_{i}\right|
        \right)
        \exp\left(-\frac{\mu^2_{i}}{2\sigma^2}\right)
        +
        \frac{1}{N_s} \sum_{i=1}^{N_s} \left|\mu_{i}\right|
        \\ & =
        \frac{1}{N_s} \sum_{i=1}^{N_s} \Delta_i \exp\left(-\frac{\mu^2_{i}}{2\sigma^2}\right)
        +
        \frac{1}{N_s} \sum_{i=1}^{N_s} \left|\mu_{i}\right|
        ,
    \end{aligned}
    \end{equation}
    where $ \Delta_i \triangleq \sqrt{\frac{2}{\pi}}\sigma - \left|\mu_{i}\right| $.
    The last expression in~\eqref{eq:mae-lb-proof-2} comprises two sums, the second of which is $J_1(\bs{\alpha}^\ast)$, by definition.
    Clearly, the optimal noiseless coefficients $\bs{\alpha}^\dagger$ satisfy $ J_1(\bs{\alpha}^\ast) \geq J_1(\bs{\alpha}^\dagger) $, yielding
    \begin{equation}\label{eq:mae-lb-proof-3}
    \begin{aligned}
        \tilde{J}_{1}(\bs{\alpha}^\ast)
        & \geq
        J_1(\bs{\alpha}^\dagger)
        +
        \frac{1}{N_s} \sum_{i=1}^{N_s} \Delta_i \exp\left(-\frac{\mu^2_{i}}{2\sigma^2}\right)
        .
    \end{aligned}
    \end{equation}
    The bound in~\eqref{eq:mae-lb-proof-3} still depends on $\alpha^\ast$ through $\mu_i$ and $\sigma$. 
    Towards avoiding this dependence, we define $\bar{\mu}_i \triangleq \max_{t=1,\dots,T} \{\left|\varphi_t(\bs{x}_i)-y_i\right|\}$ and 
    \[ \bar{\sigma} \triangleq \min_{\bs{\alpha}\in\mathbb{R}^T: \bs{1}^\top\bs{\alpha}=1} \sqrt{\bs{\alpha}^\top\bs{\Sigma}\bs{\alpha}} = \sqrt{\min_{\bs{\alpha}\in\mathbb{R}^T: \bs{1}^\top\bs{\alpha}=1} \bs{\alpha}^\top\bs{\Sigma}\bs{\alpha}}, \]
    which is known to be equivalent to the square root of the minimal generalized eigenvalue of the pair $\bs{\Sigma}$ and $\bs{1}\bs{1}^\top$.
    For normalized coefficients, i.e., $\bs{1}^\top\bs{\alpha}=1$, we can now further bound~\eqref{eq:mae-lb-proof-3} as
    \begin{equation}\label{eq:mae-lb-proof-4}
    \begin{aligned}
        \tilde{J}_{1}(\bs{\alpha}^\ast)
        & \geq
        J_1(\bs{\alpha}^\dagger)
        +
        \frac{1}{N_s} \sum_{i=1}^{N_s} \bar{\Delta}_i \exp\left(-\nicefrac{\mu^2_{i}}{2\sigma^2}\right)
        ,
    \end{aligned}
    \end{equation}
    where $\bar{\Delta}_i \triangleq \sqrt{\frac{2}{\pi}}\bar{\sigma} - \bar{\mu}_{i} $. The bound holds for normalized coefficients $\bs{\alpha}^\ast=\bs{\alpha}$ since 
    \begin{equation}\label{eq:mae-lb-proof-5}
    \begin{aligned}
        \left|\mu_i\right| = 
        \left|\bs{\alpha}^\top\bs{\varphi}(\bs{x}_i)-y_i\right| & \leq
        \bs{\alpha}^\top \left|\bs{\varphi}(\bs{x}_i)-y_i\bs{1}\right| \\ & \leq
        \max_{t=1,\dots,T} \left|\varphi_t(\bs{x}_i)-y_i\right|.
    \end{aligned}
    \end{equation}
    Since $ \bar{\sigma} \leq \sigma $, note that the $i^{\textnormal{th}}$ summand in the sum on the right-hand side of~\eqref{eq:mae-lb-proof-3} is lower bounded by $\bar{\Delta}_i \exp\left(-\frac{\bar{\mu}^2_{i}}{2\bar{\sigma}^2}\right)$ when $\bar{\Delta}_i\geq0$; 
    while for $\bar{\Delta}_i<0$ it is bounded simply by $\bar{\Delta}_i$.
    We can therefore express the lower bound as
    \begin{equation}\label{eq:mae-lb-proof-6}
    \begin{aligned}
        \tilde{J}_{1}(\bs{\alpha}^\ast)
        & \geq
        J_1(\bs{\alpha}^\dagger)
        +
        \frac{1}{N_s} \sum_{i=1}^{N_s} \bar{\Delta}_i \exp\left(- \sign'\left(\bar{\Delta}_i\right) \frac{\bar{\mu}^2_{i}}{2\bar{\sigma}^2}\right)
        ,
    \end{aligned}
    \end{equation}
    where the sign' function is defined as $\sign'(z) = 1 $ for $z\geq0$ and $\sign'(z) = 0 $ otherwise.
\end{proof}

The lower bound in Proposition~\ref{thm:mae-bagging-low-bnd-norm-coeffs} is useful when the values of $ \bar{\Delta}_i $ are mostly positive, that is, when the loss is noise-dominated (observe that $ \bar{\Delta}_i $ is a difference between a noise term and a model-error term).  
When the sum in~\eqref{eq:mae-bagging-low-bnd-norm-coeffs} is dominated by negative $ \bar{\Delta}_i $s, the bound may become loose, and we can instead use the following simpler bound.
\begin{proposition}\label{thm:mae-bagging-low-bnd-norm-coeffs-2}
Let $\bs{\alpha}^{\ast}$ be a coefficient vector that minimizes the MAE loss $\tilde{J}_1(\cdot)$. Then, 
    \begin{equation}\label{eq:mae-bagging-low-bnd-norm-coeffs-2}
        \tilde{J}_{1}(\bs{\alpha}^\ast)
        \geq
        J_1(\bs{\alpha}^\dagger),
    \end{equation}
    where $ \bs{\alpha}^\dagger $ is the optimal noiseless coefficient vector and $J_1(\cdot)$ is the noiseless MAE loss.
\end{proposition}
\begin{proof}
    First, define $ \chi_i(\bs{\alpha}) \triangleq \bs{\alpha}^\top \bs{n} +\bs{\alpha}^\top \varphi(\bs{x}_i) - y_i$. 
    Then, $ \chi_i(\bs{\alpha}) $ is a Gaussian random variable with mean $ \mu_i(\bs{\alpha}) = \bs{\alpha}^\top \varphi(\bs{x}_i)-y_i $. 
    Now, note that any random variable $ V $ satisfies $ \mathbb{E}\left[\left|V\right|\right]  = \left|\mathbb{E}\left[\left|V\right|\right]\right| \geq \left|\mathbb{E}\left[V\right]\right| $.
    Hence the first inequality in
    \begin{equation}
    \begin{aligned}
        \tilde{J}_1(\alpha^\ast)
        & =
        \frac{1}{N_s}\sum_{i=1}^{N_s} \mathbb{E}_{\bs{n}}\left[|\chi_i(\bs{\alpha}^\ast)|\right]
        \geq
        \frac{1}{N_s}\sum_{i=1}^{N_s} \left|\mathbb{E}_{\bs{n}}\left[\chi_i(\bs{\alpha}^\ast)\right]\right|
        \\ & =
        \frac{1}{N_s}\sum_{i=1}^{N_s} \left|\mu_i(\bs{\alpha}^\ast)\right|
        =
        J_1(\bs{\alpha}^\ast)
        \geq
        J_1(\bs{\alpha}^\dagger)
        ,
    \end{aligned}
    \end{equation}
    where $ \bs{\alpha}^\dagger $ is the optimal noiseless coefficient vector. The equalities are by definitions, and the last inequality holds since $\bs{\alpha}^\dagger$ minimizes $ J_1(\cdot) $.
\end{proof}

\section{Robust Gradient Boosting}\label{sec:robust-gradboost}

The ensemble bagging approach, studied in Section~\ref{sec:robust-bagging}, optimizes the aggregation coefficients of an already-trained ensemble of base regressors. A potentially stronger approach considers the existence of noise and its parameters during the training of the ensemble's base regressors themselves.
Assuming knowledge of the noise statistics during training, we devise in this section a noise-mitigating training procedure based on the popular \emph{gradient boosting (GB)} method~\cite{breiman1997arcing, friedman2001greedy, friedman2002stochastic}.

In gradient boosting, base regressors are trained sequentially: every iteration calculates an ensemble regressor $\hat{f}_{t}(\cdot)=\sum_{\tau=1}^{t}\alpha_{\tau}\varphi_{\tau}(\cdot)$ by adding a base regressor $\varphi_{t}(\cdot)$ with its coefficient $\alpha_t$. $\varphi_t(\cdot)$ is trained based on $\hat{f}_{t-1}(\cdot)$ from the previous iteration, where in GB the particular objective is that $\varphi_t(\cdot)$ will approximate the (negative) \emph{gradient} of the loss with respect to $\hat{f}_{t-1}(\cdot)$, such that it can achieve a steep reduction in the loss. Once $\varphi_t(\cdot)$ is trained, its aggregation coefficient $\alpha_t$ is chosen such that the newly formed prediction $\hat{f}_{t}(\cdot)$ minimizes the loss calculated over the training dataset.

We devise a \emph{robust} variant of gradient boosting, in which every base regressor is added into the ensemble with an aggregation coefficient that optimizes a \emph{noise-informed} loss term.
For training the base regressors, we use the standard-GB gradient estimation (without noise), such that the new variant can still benefit from the numerous algorithms available in the literature for training base-regressor models. 

A general version of the proposed algorithm, presented specifically for the $J_{p}(\bs{\alpha})$ loss, is described in Algorithm~\ref{algo:grad-boost}.
The procedure starts by initializing the first base regressor to a constant unity function $\varphi_1(\cdot) = 1$, and setting its corresponding coefficient $\alpha_1$ such that the loss over the training dataset is minimized, \emph{in expectation over the noise} (Line~\ref{alg_line:gb-alpha1}).
The $t^{\textnormal{th}}$ base regressor ($t>1$) is trained, as usual in the GB method, to approximate the negative gradient of the loss with respect to $\hat{f}_{t-1}(\cdot)$, which is denoted $\gamma^{(p)}_t$ for the loss function $J_{p}$ (Line~\ref{alg_line:gb-calc-grad}). For this gradient approximation, the objective is quadratic error, regardless of the specified loss $J_{p}$ (Line~\ref{alg_line:gb-train-phi_t}).
In case of the MSE loss ($J_2$), we get that $ \gamma^{(2)}_t(\bs{x}_i) = 2(y_i-\hat{f}_{t-1}(\bs{x}_i))$; for the MAE loss ($J_1$) we get $ \gamma^{(1)}_t(\bs{x}_i) = \sign\left(y_i- \hat{f}_{t-1}(\bs{x}_i)\right) $. Finally, the coefficient $\alpha_t$ is set to minimize the expected loss (over the noise) of this iteration's regressor $\tilde{f}_{t}(\cdot)$ (Line~\ref{alg_line:gb-calc-alpha_t}).

\begin{algorithm}
    \caption{Robust Gradient Boosting for noisy regression with $\tilde{J}_p(\cdot)$ loss}
    \label{algo:grad-boost}
    \begin{algorithmic}[1]
        \STATE \underline{Input}:$ \{(\bs{x}_i,y_i)\}_{i=1}^{N_s} $ - training dataset.
        \STATE \underline{Output}: $ \{\varphi_t(\cdot)\}_{t=1}^{T} $, $ \{\alpha_t\}_{t=1}^{T} $ - trained base regressors and aggregation coefficients.
        \STATE \label{alg_line:gb-alpha1} Initialize $\hat{f}_{1}(\cdot) \triangleq \alpha_1 \varphi_1(\cdot)$ where $\varphi_1(\cdot) \equiv 1$ and 
            \begin{equation}\label{eq:alg-gradboost-phi-0}
                \alpha_1 = \arg\min_{\alpha\in\mathbb{R}}
                \frac{1}{N_s} \sum_{i=1}^{N_s} \mathbb{E}_{\bs{n}}\left[ \vert y_i - \alpha\tilde{\varphi}_1(\bs{x}_i)\vert^p \right]
            \end{equation}
        \FOR{$t=2$ to $T$}
            \STATE \label{alg_line:gb-calc-grad} Calculate the loss negative gradient
                $
                \gamma^{(p)}_t(\cdot) = -\frac{\partial J_{p}(\hat{f}_{t-1}(\cdot))}{\partial\hat{f}_{t-1}(\cdot)}
                $
            \STATE \label{alg_line:gb-train-phi_t} Train base regressor $\varphi_t(\cdot)$ to minimize the gradient approximation error  
            \begin{equation}\label{eq:alg-gradboost-phi-t}
                \frac{1}{N_s} \sum_{i=1}^{N_s} \left( \gamma^{(p)}_t(\bs{x}_i) - \varphi_t(\bs{x}_i) \right)^2
            \end{equation}
            \STATE \label{alg_line:gb-calc-alpha_t} Denote $\omega_i^{(t)} \triangleq y_i - \sum_{\tau=1}^{t-1} \alpha_{\tau}\tilde{\varphi}_{\tau}(\bs{x}_i)$ and set $\alpha_t$ as
                \begin{equation}\label{eq:alg-gradboost-a-t}
                    \arg\min_{\alpha\in\mathbb{R}}
                    \frac{1}{N_s} \sum_{i=1}^{N_s} \mathbb{E}_{\bs{n}}\left[ \left\vert \omega_i^{(t)} - \alpha\tilde{\varphi}_t(\bs{x}_i) \right\vert^p \right]
                \end{equation}
            \STATE Set $\hat{f}_t(\cdot) = \sum_{\tau=1}^{t} \alpha_{\tau} \varphi_{\tau}(\cdot)$
        \ENDFOR
        \RETURN $\left\{\varphi_t(\cdot)\right\}_{t=1}^{T} $ and $ \{\alpha_t\}_{t=1}^{T} $
    \end{algorithmic}
\end{algorithm}

We now derive the optimal aggregation coefficients in~\eqref{eq:alg-gradboost-phi-0} and~\eqref{eq:alg-gradboost-a-t} for MSE loss. 

\begin{proposition}\label{prop:mse-gradboost-alpha0-alphat}
    Let $\bs{\alpha}^\ast=(\alpha_1,\dots,\alpha_T)\in\mathbb{R}^T$ be the vector of aggregation coefficients obtained by Algorithm~\ref{algo:grad-boost} for MSE loss ($p=2$).
    Given that the noise random vector has mean $\bs{0}$ and covariance matrix $\bs{\Sigma}$, the coefficients are given by
    \begin{equation}\label{eq:prop-mse-gradboost-alpha1}
    \begin{aligned}
       \alpha_1 = \frac{1}{(1+\sigma_1^2)N_s} \sum_{i=1}^{N_s} y_i,
    \end{aligned}
    \end{equation}
    and then for $t>1$ recursively by 
    \begin{equation}\label{eq:prop-mse-gradboost-alpha0-alphat}
    \begin{aligned}
        \alpha_t = \frac
                         { \frac{1}{N_s} \sum_{i=1}^{N_s} \varphi_t(\bs{x}_i) \left( y_i - \hat{f}_{t-1}(\bs{x}_i) \right) + \sum_{\tau=1}^{t-1} \alpha_{\tau} \bs{\Sigma}_{t,\tau} }
                         { \sigma^2_t + \frac{1}{N_s} \sum_{i=1}^{N_s} \varphi_t^2(\bs{x}_i) }.
    \end{aligned}
    \end{equation}
\end{proposition}
\begin{proof}
    The first coefficient, $\alpha_1$, is set in~\eqref{eq:alg-gradboost-phi-0} to minimize  $\frac{1}{N_s}  \sum_{i=1}^{N_s} \mathbb{E}_{\bs{n}} \left[ \left( y_i-\alpha(1+n_1) \right)^2 \right]$. Since $\mathbb{E}_{\bs{n}}\left[n_1\right] = 0$ and $\mathbb{E}_{\bs{n}}\left[n_1^2\right] = \sigma_1^2$, the derivative with respect to $\alpha$ of the expectation becomes $\frac{2}{N_s}\sum_{i=1}^{N_s} \left( y_i - \alpha (1+\sigma_1^2) \right)$, and equating to zero gives the right-hand side of~\eqref{eq:prop-mse-gradboost-alpha1}.

    At step $t>1$, $\alpha_t$ minimizes in~\eqref{eq:alg-gradboost-a-t} the expectation 
    {\footnotesize
    \begin{equation}\label{eq:mse-gradboost-loss-proof-eq-1}
    \begin{aligned}
        & \frac{1}{N_s} \sum_{i=1}^{N_s} \mathbb{E}_{\bs{n}} \left[\left( y_i - \tilde{f}_{t-1}(\bs{x}_i) - \alpha\left(\varphi_t(\bs{x}_i)+n_t\right) \right)^2 \right]
        \\ & =
       \frac{1}{N_s} \sum_{i=1}^{N_s} \left\{ \mathbb{E}_{\bs{n}}\left[\left( y_i - \tilde{f}_{t-1}(\bs{x}_i) \right)^2\right] \right.\\& \left. - 2 \mathbb{E}_{\bs{n}}\left[\left( y_i - \tilde{f}_{t-1}(\bs{x}_i) \right) \alpha\left(\varphi_t(\bs{x}_i)+n_t\right)\right] + \mathbb{E}_{\bs{n}} \left[\alpha^2\left(\varphi_t(\bs{x}_i)+n_t\right)^2\right] \right\}
        \\ & =
        \frac{1}{N_s} \sum_{i=1}^{N_s} \mathbb{E}_{\bs{n}} \left[\left(y_i-\tilde{f}_{t-1}(\bs{x}_i)\right)^2\right] 
        \\& - 
        \frac{2 \alpha}{N_s} \sum_{i=1}^{N_s} \left\lbrace \left(y_i-\hat{f}_{t-1}(\bs{x}_i) \right)\varphi_t(\bs{x}_i) + \mathbb{E}_{\bs{n}}\left[\tilde{f}_{t-1}(\bs{x}_i) n_t\right] \right\rbrace
        \\& +
        \frac{\alpha^2}{N_s} \sum_{i=1}^{N_s} \left( \varphi_t^2(\bs{x}_i) + \mathbb{E}_{\bs{n}}\left[n_t^2\right] \right),
    \end{aligned}
    \end{equation}
    }
    where in the second equality we used the facts $\mathbb{E}_{\bs{n}}\left[n_t\right] = 0$ and $\mathbb{E}_{\bs{n}}\left[\tilde{f}_{t-1}(\bs{x}_i)\right] = \hat{f}_{t-1}(\bs{x}_i)$. Since $\mathbb{E}_{\bs{n}}\left[n_t^2\right] = \sigma_t^2$ and $\mathbb{E}_{\bs{n}}\left[n_{\tau}n_t\right] = \bs{\Sigma}_{t,\tau}$, we can simplify a term in the above expression:
    \begin{equation}
    \begin{aligned}
        \mathbb{E}_{\bs{n}}\left[\tilde{f}_{t-1}(\bs{x}) n_t\right]
        & =
        \mathbb{E}_{\bs{n}}\left[ \sum_{\tau=1}^{t-1} \alpha_{\tau} \left(\varphi_{\tau}(\bs{x})+n_{\tau}\right) n_t \right]
        \\ & =
        \mathbb{E}_{\bs{n}}\left[ \sum_{\tau=1}^{t-1} \alpha_{\tau} n_{\tau} n_t \right]
        =
        \sum_{\tau=1}^{t-1} \alpha_{\tau} \bs{\Sigma}_{t,\tau}.
    \end{aligned}
    \end{equation}
    Now we can take the simplified~\eqref{eq:mse-gradboost-loss-proof-eq-1} and calculate the derivative with respect to $\alpha$
    \begin{equation}\label{eq:boost_mse_deriv}
    \begin{aligned}
        -\frac{2}{N_s} \sum_{i=1}^{N_s} & \varphi_t(\bs{x}_i) \left( y_i - \hat{f}_{t-1}(\bs{x}_i) \right)
        -
        2 \sum_{\tau=1}^{t-1} \alpha_{\tau} \bs{\Sigma}_{t,\tau}
        \\ & +
        2\alpha \left( \sigma^2_t + \frac{1}{N_s} \sum_{i=1}^{N_s} \varphi_t^2(\bs{x}_i) \right).
    \end{aligned}
    \end{equation}
    
    Equating~\eqref{eq:boost_mse_deriv} to $0$ finally gives~\eqref{eq:prop-mse-gradboost-alpha0-alphat}.
    Since the MSE loss is a convex function, the obtained solution is the minimum.
\end{proof}

Proposition~\ref{prop:mse-gradboost-alpha0-alphat} provides, for MSE loss ($p=2$), closed-form solutions to the coefficient optimizations of Algorithm~\ref{algo:grad-boost}, and thus makes the algorithm simple to implement and fast to run. 
The same algorithm may be employed for MAE loss ($p=1$) as well, while replacing in~\eqref{eq:alg-gradboost-a-t} the closed-form expressions for $\alpha_t$ with a search-based minimization algorithm. For example, the (scalar) derivative of the expected MAE can be derived in a similar way to Proposition~\ref{prop:mae-bagg-gradient}, then employing it in a descent-based search of $\alpha_t$.

\section{Experimental results}\label{sec:experiments}
\begin{figure*}[ht]
    \centering
    \begin{subfigure}{0.475\textwidth}\label{fig:loss-gain-mse-equi-variance}
        \centering
        \includegraphics[width=\textwidth]{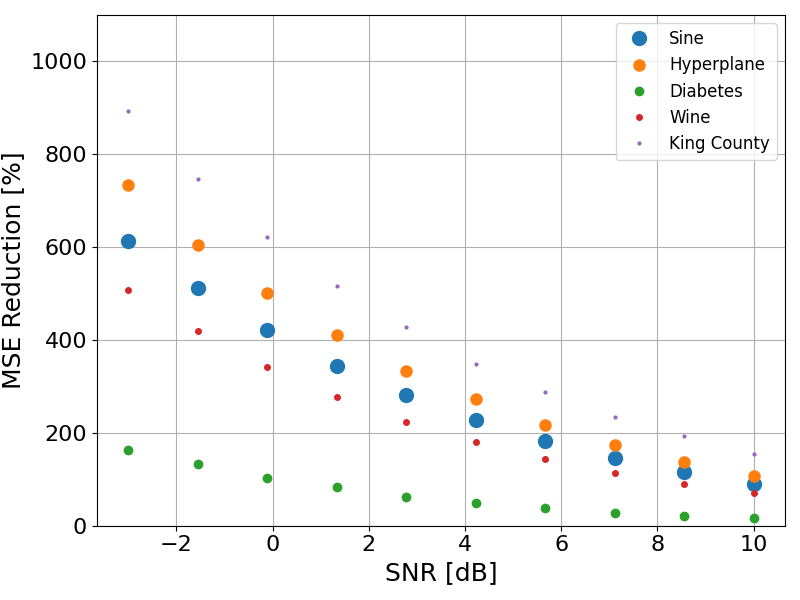}
        \caption{TEM performance, equi-variance noise.}
    \end{subfigure}
    \hfill
    \begin{subfigure}{0.475\textwidth}\label{fig:loss-gain-mse-noisy-subset}
        \centering
        \includegraphics[width=\textwidth]{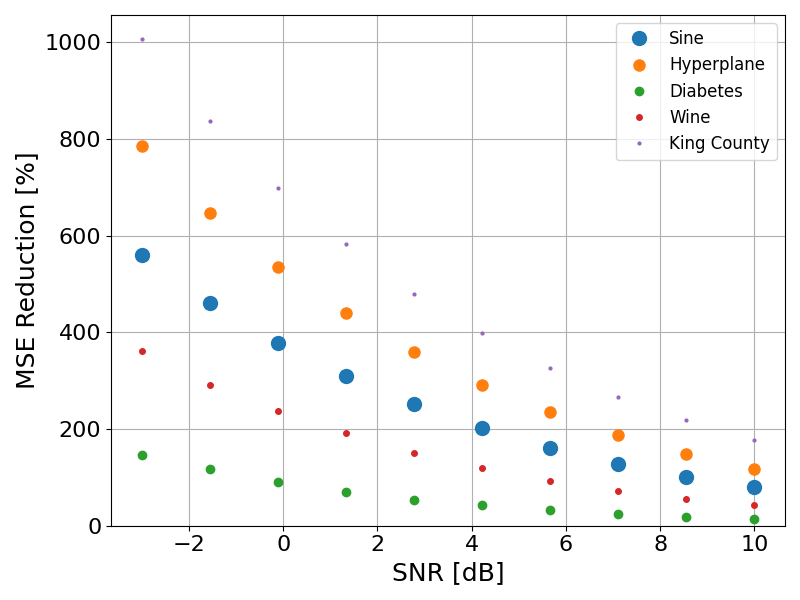}
        \caption{TEM performance, noisier-subset with $m=2$, $a=20$.}
    \end{subfigure}
    \caption{MSE reduction of TEM bagging ensembles relative to the prior method GEM, as a function of $\mathrm{SNR}$.}
    \label{fig:loss-gain-mse}
\end{figure*}

In this section, we give an empirical evaluation of the advantages of the proposed robust methods in comparison to non-robust classical ensemble methods. We also show empirical results that improve our understanding of the proposed methods and the considerations for their practical use. Toward these ends, we experimented with various datasets and several noise conditions. Besides the sinusoidal sum target function shown in Section~\ref{sec:motivating-example}, we used another synthetic function and three real-world datasets:
\begin{itemize}
    \item Hyperplane: a linear combination $f(\bs{x}) = \bs{c}^\top \bs{x}$ where $\bs{x}\in\mathbb{R}^{d}$, $d=3$ and $\bs{c}$ is a vector of random real numbers.
    \item UCI diabetes patient records~\cite{kahn1994diabetes}: the dataset consists of $d=10$ variables and $N_{s}=442$ samples.
    \item White-wine quality~\cite{cortez2009wine}: the dataset consists of $d=11$ variables and $N_{s}=4898$ samples.
    \item King County house prices: the dataset consists of $d=20$ variables and $N_{s}=21613$ samples.
\end{itemize}
In the following experiments, each synthetic dataset comprised $N_{s}=1000$ samples and was generated by summing the target function with independent and identically distributed random Gaussian measurement error $\epsilon \sim \mathcal{N}(0,0.01)$ (not to be confused with the channel noise, which is the main object of interest). In addition, for convenience all datasets were standardized to have zero mean and unit standard deviation ($\varepsilon_y=1$).

During inference on each test subset, the individual predictions of the base regressors are added with noise before they are aggregated to produce the final prediction. For training and inference we followed the \emph{$k$-fold cross-validation} procedure, according to which $k$ equally-sized disjoint subsets of data samples were generated from the dataset. The following procedure is repeated for each such subset (fold): using the union of the complement $k-1$ subsets, train an ensemble and set the aggregation coefficients; then evaluate the prediction performance on the subset. The performance of each fold's ensemble is averaged over $100$ noise realizations. 

While this paper's results apply to general noise distributions, we focus in this section on two representative settings: a homogeneous setting in which all channels have equal noise variance; and a heterogeneous setting with two channel types, such that a part of the ensemble communicates through noisier channels compared to the remaining part. The corresponding noise covariance matrices are:
\begin{itemize}
    \item \emph{Equi-variance noise profile}: A diagonal matrix with equal non-zero elements, i.e., $\bs{\Sigma} = \sigma^2\bs{I}$ where $\bs{I}$ is the $T \times T$ identity matrix and 
    \begin{equation}
        \sigma^2 = \frac{\varepsilon_y}{\mathrm{SNR}}.
    \end{equation}
    \item \emph{Noisier-subset noise profile}: Given a fraction $1/m$ of factor $a>1$ noisier channels, a diagonal matrix whose non-zero elements are a permutation of the vector
    \begin{equation*}
        \left[ \underbrace{a\sigma^2,\dots,a\sigma^2}_{T/m \text{ times}}, 
        \underbrace{\sigma^2,\dots,\sigma^2}_{T-T/m \text{ times}}
        \right].
    \end{equation*}
\end{itemize}
In this case, according to~\eqref{eq:snr-def},
\begin{equation}
    \sigma^2 
    = 
    \frac{T\cdot\varepsilon_y}{\left( \frac{T}{m} a + \left(1-\frac{1}{m}\right)T \right) \cdot \mathrm{SNR}}
    = 
    \frac{m\cdot\varepsilon_y}{\left( a + m-1 \right) \cdot \mathrm{SNR}}
    .
\end{equation}
Specifically, we used $m=2$ in our simulations, hence the diagonal elements are $\sigma^2, a\sigma^2, \sigma^2, a\sigma^2, \dots$ and
\begin{equation}
    \sigma^2 = \frac{\varepsilon_y}{(1+\frac{a-1}{2})\cdot\mathrm{SNR}}.
\end{equation}

\subsection{Robust MSE-optimal bagging ensembles}
We start with robust bagging ensembles optimized for MSE loss (Section~\ref{subsec:bagging-mse}).
Our focus in this part is the trade-off ensemble method (TEM), in particular, using Theorem~\ref{thm:tem-mse} with $\lambda=1$ to find the $\bs{\alpha}$ that minimizes~\eqref{eq:l2-bagging} in Problem~\ref{prob:l2-bagging}. 
We trained an ensemble of $T=32$ base regressors, each of which is a decision tree (with varying depth, depending on the dataset). Per the bagging approach, each base regressor was trained on a random subset of the fold's training set. After obtaining the ensemble, we used~\eqref{eq:tem-coeff} to calculate the MSE-optimal aggregation coefficients.  

We measured the performance of the TEM coefficients by evaluating the robustness gain, i.e., the \emph{MSE reduction} compared to the GEM coefficients in~\eqref{eq:gem-coeff} (which do not take the noise into account). Formally, we measure

{\small
\begin{equation}\label{eq:err-gain}
    100 \cdot \frac{
        \sqrt[\leftroot{-3}\uproot{3}]{
            \mathbb{E}\left[\left|y-\bs{\alpha}_{\textnormal{GEM}}^\top\tilde{\bs{\varphi}}(\bs{x})\right|^2\right]
        }
        -
        \sqrt[\leftroot{-3}\uproot{3}]{
            \mathbb{E}\left[\left|y-\bs{\alpha}^\top_{\textnormal{TEM}}\tilde{\bs{\varphi}}(\bs{x})\right|^2\right]
        }
    }
    {
    \sqrt[\leftroot{-3}\uproot{3}]{
        \mathbb{E}\left[\left|y-\bs{\alpha}_{\textnormal{GEM}}^\top\bs{\varphi}(\bs{x})\right|^2\right]
    }
    }
    \left[\%\right],
\end{equation}
}

where $\bs{\alpha}_{\textnormal{GEM}}$ and $\bs{\alpha}_{\textnormal{TEM}}$ are the GEM and TEM coefficients (respectively), and $\tilde{\bs{\varphi}}$ denotes the noisy version of the bagging ensemble's base-regressor functions (note the noiseless $\bs{\varphi}$ in the denominator). The empirical expectations in~\eqref{eq:err-gain} are calculated by averaging over both the multiple folds and the noise realizations. The results are plotted in Fig.~\ref{fig:loss-gain-mse} as a function of SNR, for all five datsets. 
It can be seen that for both noise profiles, TEM offers significant performance advantages: reaching MSE reductions of $200\%$ to $1000\%$ at lower SNRs, depending on the dataset. The advantage decreases as the SNR increases. The MSE achieved by TEM is only slightly degraded compared to a baseline noiseless ensemble (not shown in the figures).
It may initially counter intuition that robust coefficients are needed even in the equi-variance case (where all channels are identical), but the empirical results show that this need in optimization is crucial. The intuitive explanation of this may be that differences in quality (in terms of model error) between regressors map to different coefficient assignments depending on the noise intensity.

Next in Fig.~\ref{fig:motivation-coef_hist} and~\ref{fig:motivation-coef_bars}, we take a refined view into the TEM optimal solutions. In Fig.~\ref{fig:motivation-coef_hist}, we plot the distributions of the aggregation coefficients obtained for the two noise profiles tested in Fig.~\ref{fig:loss-gain-mse}. Each plot in Fig.~\ref{fig:motivation-coef_hist} is a histogram of the coefficients assigned to $T=20$ regressors in the equi-variance (left) and noisier-subset (right) profiles for the sinusoidal sum dataset; each color/pattern represents a different fold. It can be observed that the equi-variance coefficients (left) are centered around the average of $0.05$ ($=1/T$), while in the noisier-subset case (right) the histogram is more spread out and bi-modal: the higher coefficients are for the regressors that have good channels, and those of the noisier channels are closer to zero. 

Fig.~\ref{fig:motivation-coef_bars} provides a deeper understanding of the relationship between the individual base-regressor's prediction quality (error and noise) and its coefficient assignment. The left bars (grey) describe an ensemble of $T=20$ base regressors with a synthetic linearly-decreasing profile of model-error (top is best). The corresponding TEM aggregation coefficients are depicted by the right bars for three noise settings: noiseless (blue), weak noise (orange) and high noise (green). When there is no noise, the coefficients vary considerably depending (inversely proportional) on the regressor's model error. In contrast, when the noise level is highest, the coefficients are close to uniform, aiming to average out the noise. The weak noise exhibits coefficients that give a compromise between these two competing considerations.

\begin{figure}[ht]
    \centering
    \includegraphics[width=1.0\linewidth]{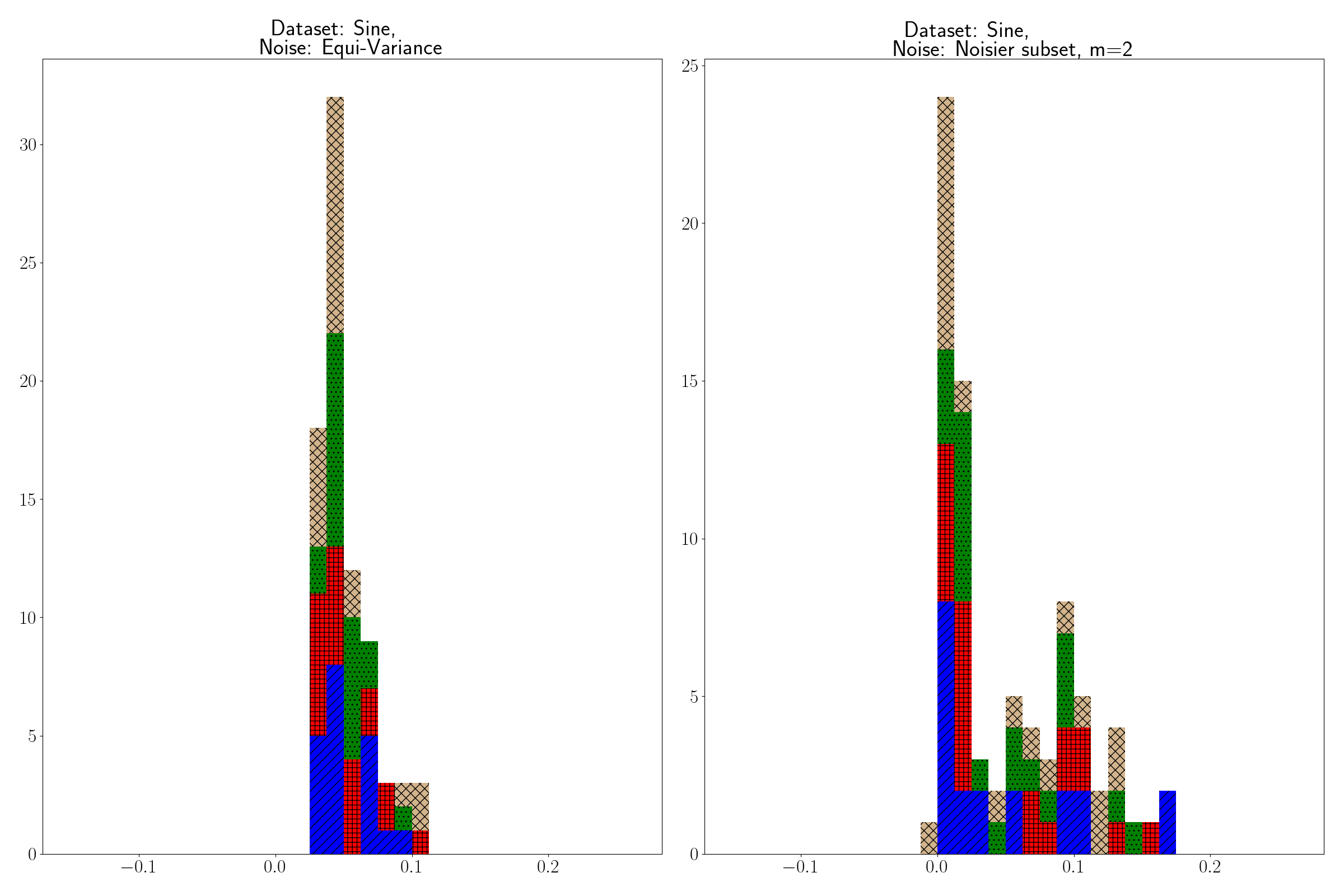}
    \caption{Histogram of TEM (MSE-optimal) aggregation coefficients for the sinusoidal sum dataset with equi-variance noise (left) and $m=2$ noisier-subset profile (right). The colors/patterns depict different folds of the training data.}
    \label{fig:motivation-coef_hist}
\end{figure}
\begin{figure}[ht]
    \centering
    \includegraphics[width=1.0\linewidth]{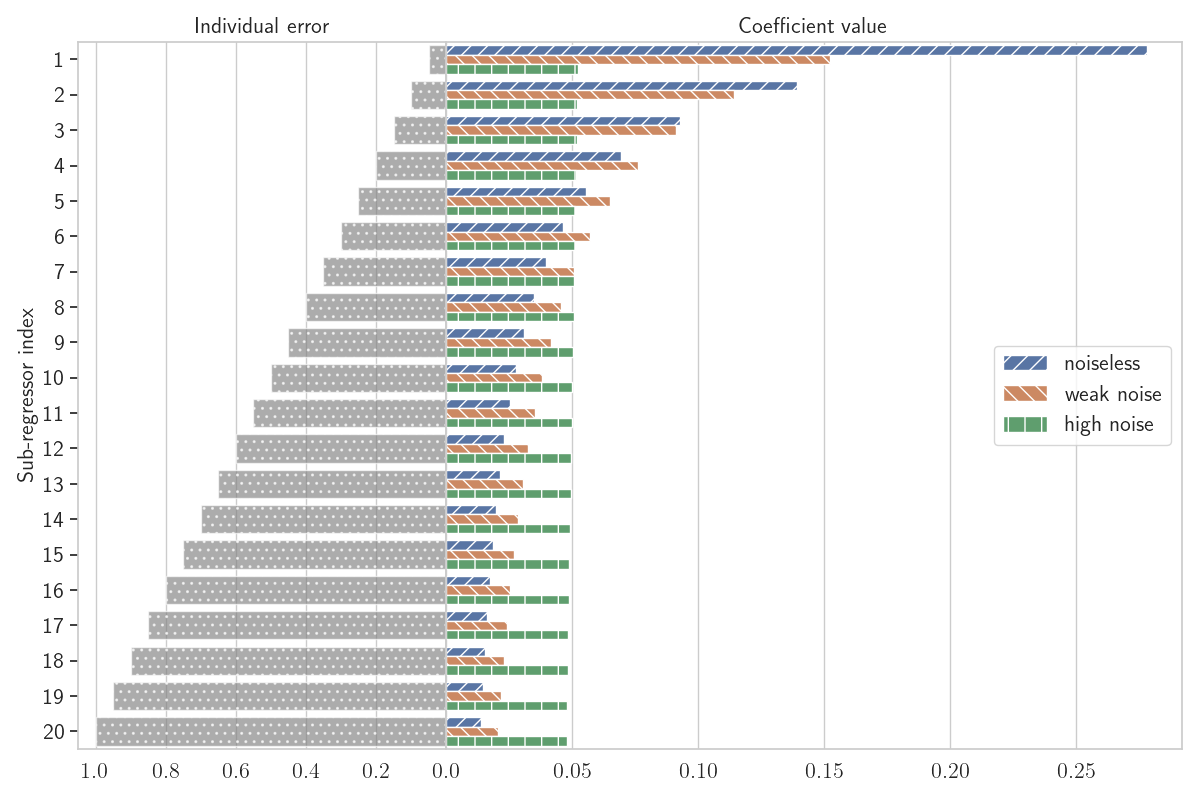}
    \caption{TEM (MSE-optimal) aggregation coefficients (right) calculated with heterogeneous base-regressor model errors (left) for three equi-variance noise settings: noiseless (blue `/'), weak noise (orange `\textbackslash'), and high noise (green `\textbar').}
    \label{fig:motivation-coef_bars}
\end{figure}

Recall from Section~\ref{subsec:bagging-mse} that the TEM optimization includes the parameter $\lambda$ that sets the relative importance of the channel noise in the total loss term $\tilde{J}_{2}^{(\lambda)}$. One possible use of this parameter is to address uncertainties in the noise model. For example, consider the case where noise exists in the channels only intermittently at some fraction of the inference instances. Instead of attempting to solve a new optimization problem for each such case, we can capture them using the flexibility of $\lambda$ within the proposed optimization. In Fig.~\ref{fig:tem-lambda-sweep} we plot the MSE obtained by TEM when optimized for equi-variance noise with parameter $\sigma$, but tested with only \emph{half} of the inference instances having noise.  The plot shows the average MSE (over both types of instances) as a function of $\lambda$, when  $\mathrm{SNR}=-6$dB for the noisy instances. It can be seen that there is an intermediate value of $0<\lambda<1$ that minimizes the MSE, and that this value depends on the dataset. We can think of this optimal $\lambda$ as measuring the sensitivity of the trained ensemble to the channel noise.

\begin{figure}[ht]
    \centering
    \includegraphics[width=0.5\textwidth]{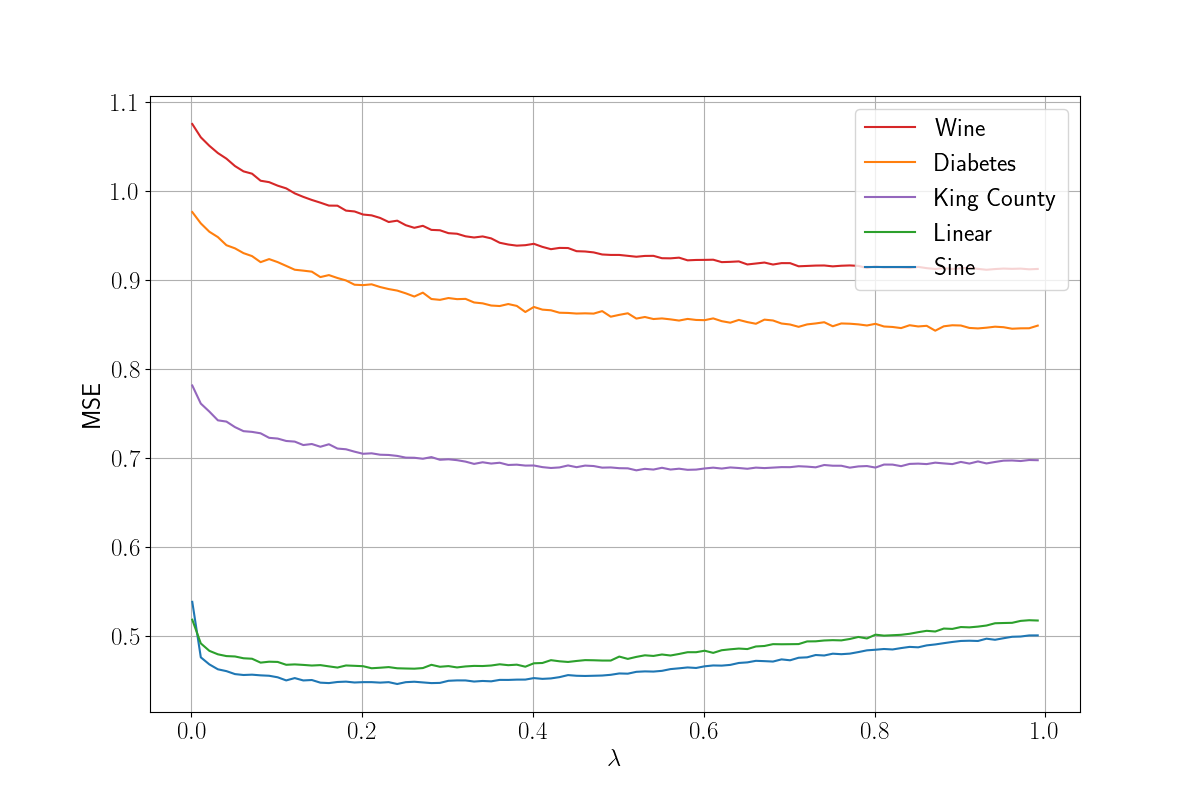}
    \caption{MSE of TEM as a function of $\lambda$ for partially noisy channels.}
    \label{fig:tem-lambda-sweep}
\end{figure}

\subsection{Robust MAE-optimized bagging ensembles}
The focus in this sub-section moves to evaluating the performance of \emph{MAE-optimized} robust ensembles (from Section~\ref{subsec:bagging-mae}). 
To this end, we use Algorithm~\ref{algo:grad-dec-bagg} with the gradient expression from Proposition~\ref{prop:mae-bagg-gradient} to MAE-optimize the coefficients of an ensemble of $T=8$ decision-tree base regressors. In Fig.~\ref{fig:loss-gain-mae}, we compare the MAE achieved by the MAE-robust coefficients relative to the non-robust coefficients calculated by Algorithm~\ref{algo:grad-dec-bagg} modified to noise-free MAE with the gradient expression in~\eqref{eq:non-robust-bagging-mae-grad}. The results are shown for the five datasets and for both noise profiles. The best way to examine these results is starting from a point at the right end of the x-axis, and following the curves left-ward while the robust and non-robust diverge. It can be seen that significant gaps open between the curves, and that the robust MAE curves are much more flat in their increase as the SNR decreases. The plots reveal interesting dependencies on the dataset: for some (as Diabetes and Wine), the non-robust ensembles are more sensitive to noise  than for others, for some (as Sine), the robust ensembles deal better with noisier-subset than with equi-variance.

\begin{figure*}[ht]
    \centering
    \begin{subfigure}{0.475\textwidth}
        \centering
        \includegraphics[width=\textwidth]{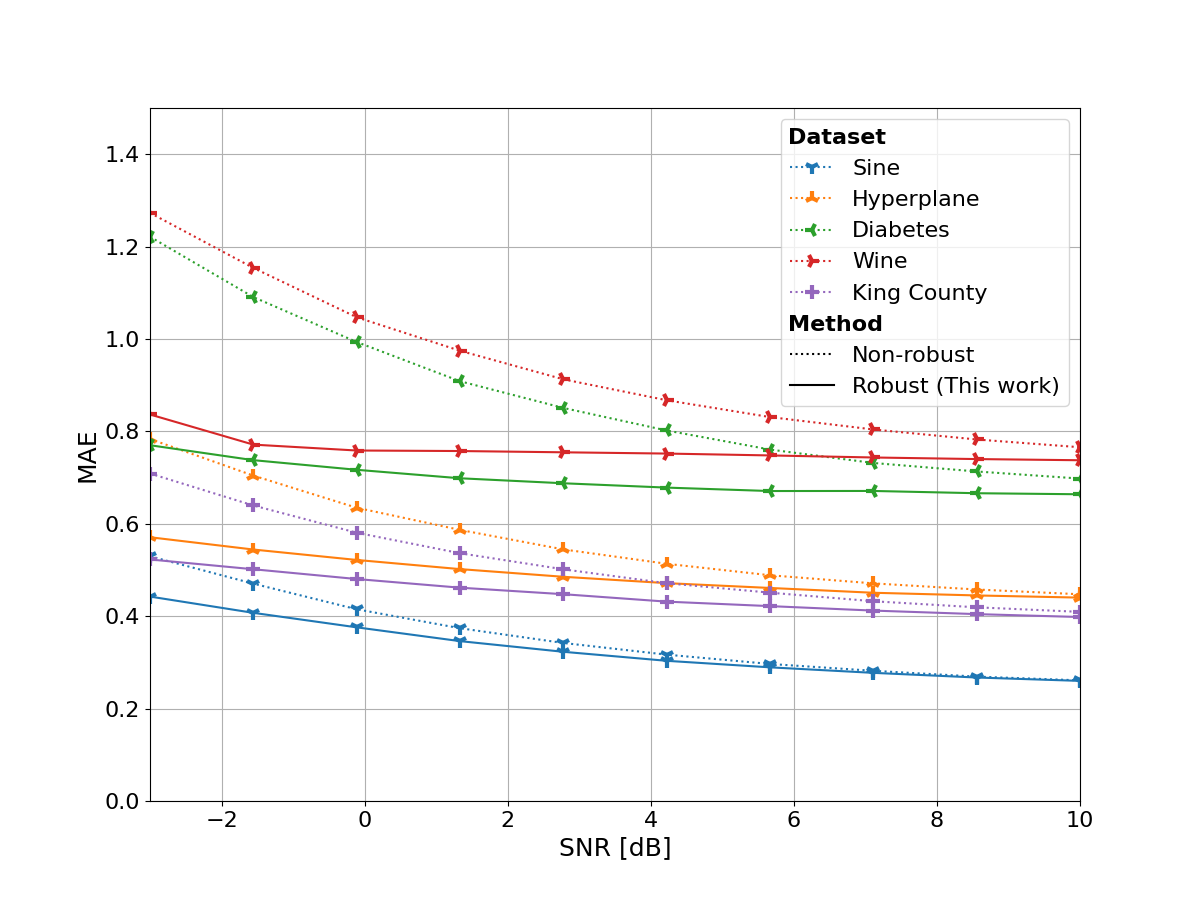}
        \caption{MAE-optimized, equi-variance noise.}
    \end{subfigure}
    \hfill
    \begin{subfigure}{0.475\textwidth}
        \centering
        \includegraphics[width=\textwidth]{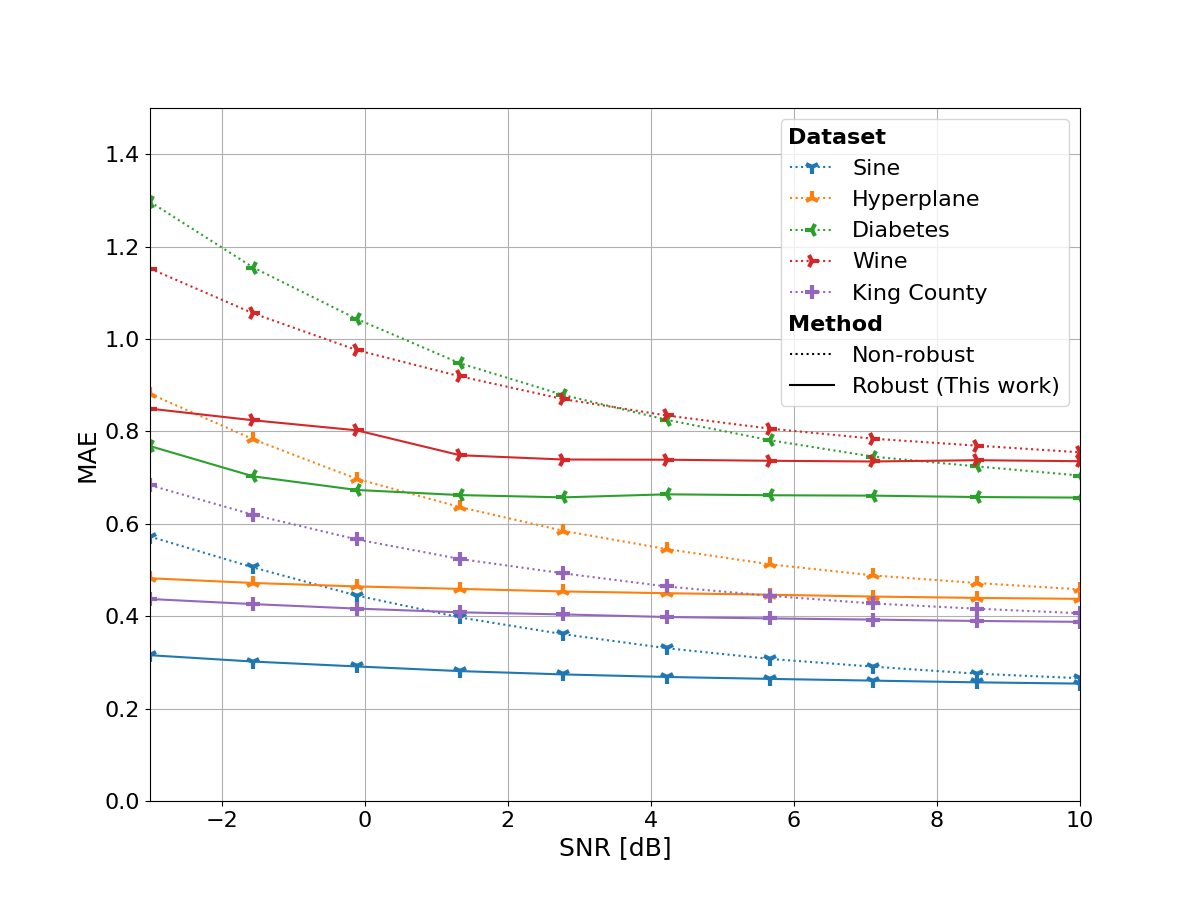}
        \caption{MAE-optimized, noisier subset with $m=2$, $a=20$.}
    \end{subfigure}
    \caption{Comparison of MAE for robust (Algorithm~\ref{algo:grad-dec-bagg}) and non-robust bagging ensembles as a function of $\mathrm{SNR}$.}
    \label{fig:loss-gain-mae}
\end{figure*}

In addition to seeing the actual MAE resulting from the proposed Algorithm~\ref{algo:grad-dec-bagg}, it is useful to examine the limits of the problem by plotting the MAE bounds derived in Section~\ref{subsec:mae-bounds}. We present in Fig.~\ref{fig:mae-bounds-even-noiseless} the lower and upper bounds on MAE for the five datasets, with the noisier-subset profile. The plotted lower bound is the maximum of the bounds in Propositions~\ref{thm:mae-bagging-low-bnd-norm-coeffs} and~\ref{thm:mae-bagging-low-bnd-norm-coeffs-2}; while the plotted upper bound is the minimum of the two bounds induced by Proposition~\ref{prop:MAE_ub} in Eq.~\eqref{eq:mae-bem-upper-bnd} and~\eqref{eq:mae-bagg-norm-upper-bnd}. Note that for the lower bounds we used a coefficient vector  optimized for a noiseless setting ($\bs{\Sigma}=\bs{0}$) as the optimal noiseless coefficient vector $\bs{\alpha}^{\dagger}$ in~\eqref{eq:mae-bagging-low-bnd-norm-coeffs} and~\eqref{eq:mae-bagging-low-bnd-norm-coeffs-2}. The bounds are tight for high $\mathrm{SNR}$ (essentially coinciding on the noiseless error); for a general SNR they together mark the range of MAE values that may be relevant for future optimization algorithms.

\begin{figure}[ht]
    \centering
    \includegraphics[width=0.9\linewidth]{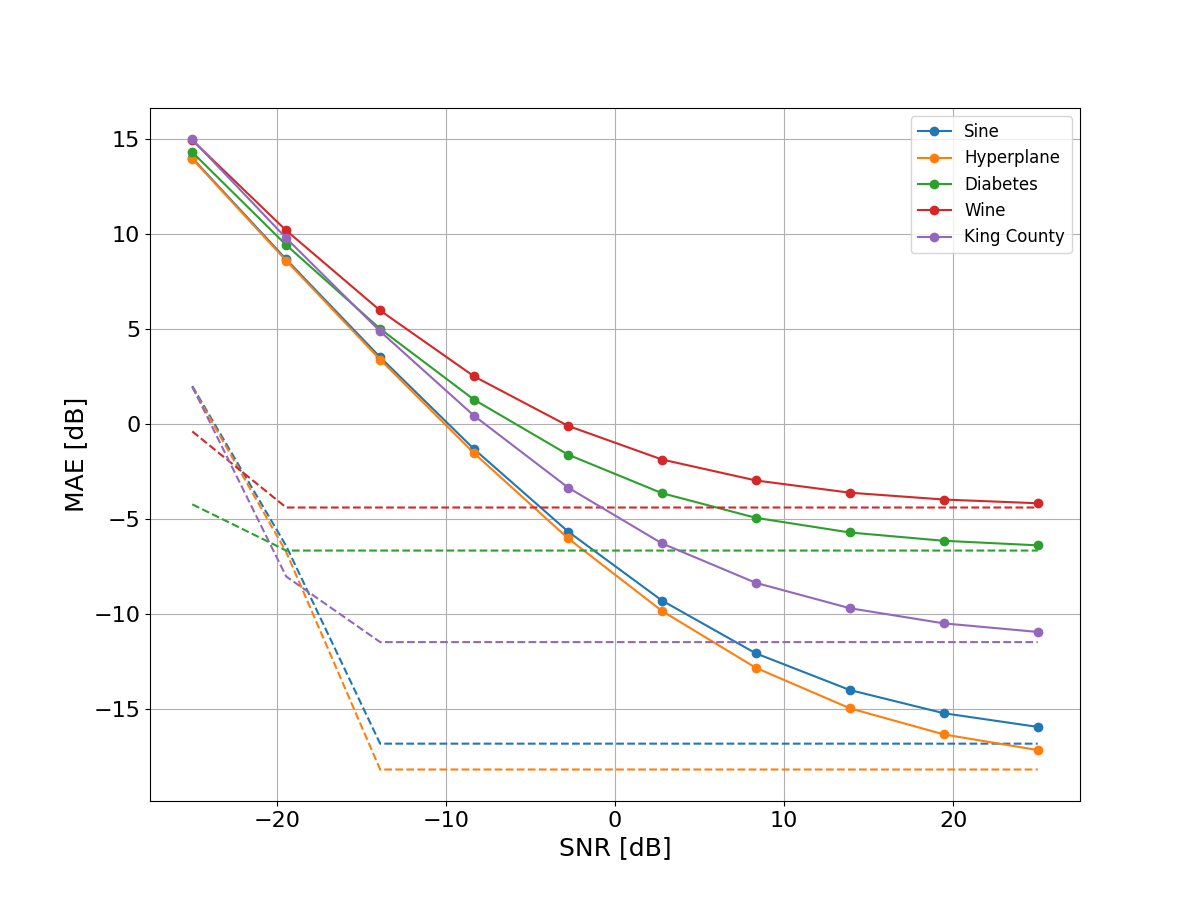}
    \caption{Lower and upper bounds on the optimal MAE (in dB, as $10\cdot\log_{10}(\cdot)$) as a function of $\mathrm{SNR}$. Noise profile is ``noisier subset'' with $m=2$, $a=20$.}
    \label{fig:mae-bounds-even-noiseless}
\end{figure}

\subsection{Robust training using gradient boosting}

In this sub-section we evaluate the performance of robust gradient-boosting ensembles developed in Section~\ref{sec:robust-gradboost}. Without noise, gradient boosting has the advantage (in particular, over bagging) of regressors that adaptively refine the error of the prior iterations. With noise, however, this advantage may turn into a disadvantage of increased sensitivity to noise in the more dominant regressors of the ensemble.

In the experiment, we pick two datasets, one natural and one synthetic, and train ensembles for them using Algorithm~\ref{algo:grad-boost} for the MSE loss. Recall that part of the algorithm is setting the aggregation coefficients using Proposition~\ref{prop:mse-gradboost-alpha0-alphat}. As a comparison baseline, we use the standard gradient-boosting (GB) algorithm that sets the aggregation coefficients to minimize the training error without considering the noise. We vary the ensemble size $T$, where each regressor in the ensemble is a depth-$1$ decision tree.  

Fig.~\ref{fig:loss-gain-even-noiseless} depicts for $\mathrm{SNR}=18\mathrm{dB}$ the (root of the) MSE as a function of $T$ for the Diabetes dataset (upper plots) and Sine dataset (lower plots). We show results for the two noise profiles: equi-variance (left) and noisier-subset (right). The first thing we notice in the plots is the behavior of the non-robust baseline algorithm: while without noise (Noiseless, black) the (root) MSE decreases as $T$ grows, with noise (Noisy GB, red) the baseline exhibits the opposite trend of error \emph{increasing} with $T$ (!) In contrast, the Robust GB (blue) curve maintains the decreasing trend even with noise. A fourth curve (Noiseless Robust GB, purple) shows that the robust ensemble desirably performs better without noise than with noise. For the Diabetes dataset this performance is the same as the standard noiseless gradient boosting.
These features underline the attractiveness of the robust ensembles: ensuring correct behavior in noisy settings without severely compromising the noiseless performance.

\begin{figure*}[ht]
    \centering
    \begin{subfigure}{0.485\textwidth}
        \centering
        \includegraphics[width=\textwidth]{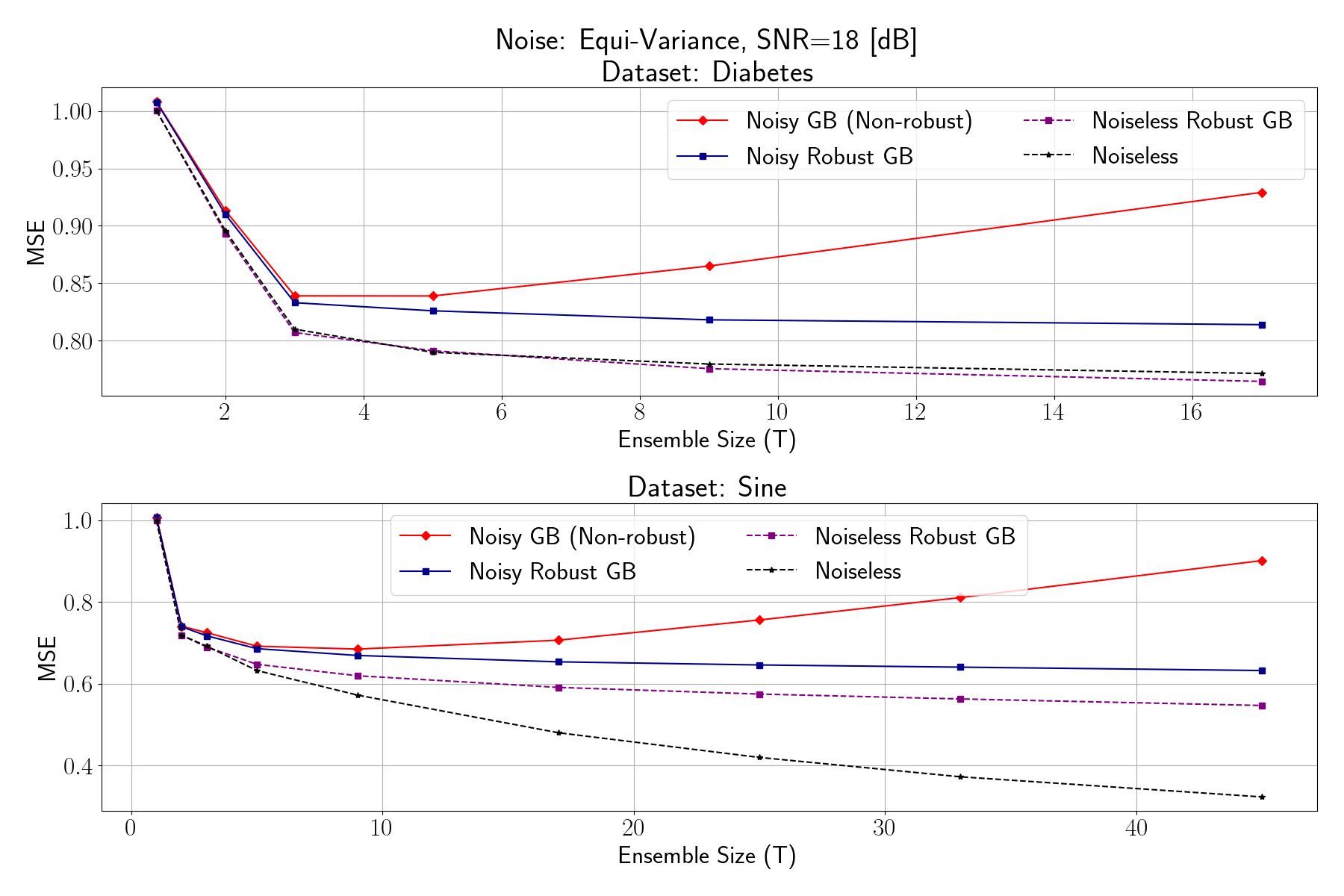}
        \caption{Equi-variance noise.}
    \end{subfigure}
    \hfill
    \begin{subfigure}{0.485\textwidth}
        \centering
        \includegraphics[width=\textwidth]{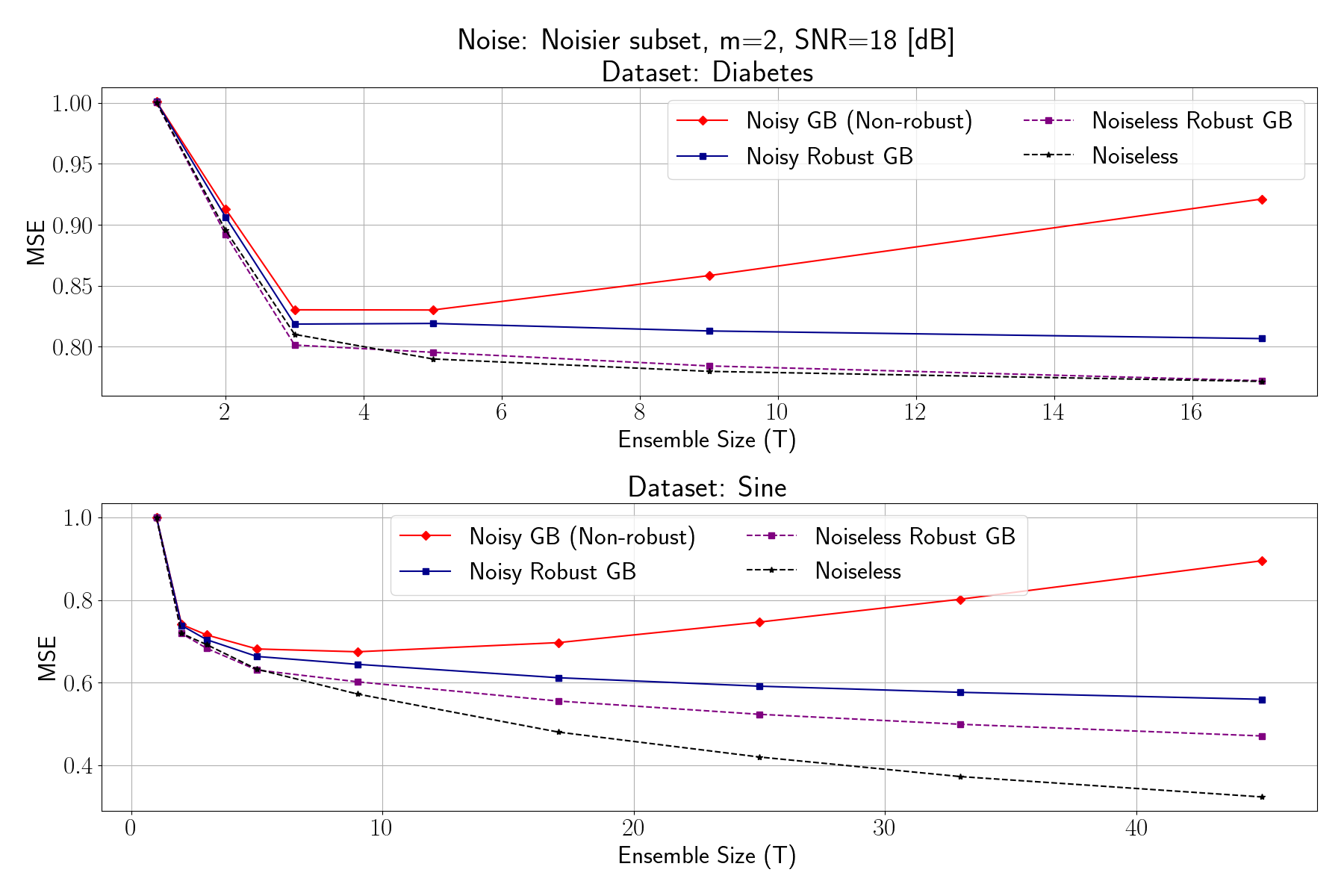}
        \caption{Noisier subset with $m=2$, $a=20$.}
    \end{subfigure}
    \caption{Comparison of (root) MSE obtained by robust and standard (``non-robust'') gradient boosting (GB) ensembles at $\mathrm{SNR}=18$dB as a function of the ensemble size $T$, for Diabetes (upper) and Sine (lower) datasets. The (root) MSEs achieved by the ensembles without noise (``Noiseless'' and ``Noiseless robust'') are shown for reference.}
    \label{fig:loss-gain-even-noiseless}
\end{figure*}

\section{Conclusion}\label{sec:conclusion}

This work presents robust inference with regression ensembles operating over a noisy distributed setting. As a closed-form expression for MAE-optimal aggregation coefficients is unavailable, an immediate future-work direction is the derivation of simpler/tighter lower and upper MAE bounds toward easier optimization of the coefficients. In addition, we conjecture that tighter guarantees on the expected loss may also be obtained by assuming specific statistical properties of the noise.

Another natural extension of this work is to address generalized aggregation schemes that consider non-scalar target functions $f(\bs{x})$ which appear in many real-world problems.
In addition, this work can also be extended to additional performance criteria (e.g., $p=0$, $p>2$ or a linear combination of model error and aggregated noise in $p\neq2$), additional noise regimes (e.g., quantization noise) and additional training, and aggregation methods (e.g., Stacking~\cite{breiman1996stacked}).


{\appendices

\bibliographystyle{IEEEtran}
\bibliography{biblio.bib}


 
%

\end{document}